%% file: main.tex
\documentclass[runningheads]{llncs}

\usepackage{eccv}

\usepackage{eccvabbrv}

\usepackage{graphicx}
\usepackage{booktabs}

\usepackage[accsupp]{axessibility}  

\usepackage{hyperref}

\usepackage{orcidlink}

\usepackage{algorithm, algorithmic}

\usepackage{multirow, graphicx, booktabs, colortbl}

\usepackage{pgffor} 

\newcommand{\ours}{CAI }

\begin{document}

\title{See Only When Needed: Context-Aware Attention Intervention for Mitigating Hallucinations in LVLMs}

\titlerunning{Context-Aware Attention Intervention for Mitigating Hallucinations}


\author{Yuqing Lei\inst{1}\orcidlink{0009-0009-4669-3774} \and 
        Wenbo Lyu\inst{1}\orcidlink{0009-0004-8962-0448} \and 
        Yingjun Du\inst{2}\orcidlink{0000-0001-7537-6457} \and 
        Xiantong Zhen\inst{3}\orcidlink{0000-0001-5213-0462} \and 
        Cees G.M. Snoek\inst{2}\orcidlink{0000-0001-9092-1556} \and 
        Ling Shao\inst{1}\orcidlink{0000-0002-8264-6117}\thanks{Corresponding author: ling.shao@ieee.org}}

\authorrunning{Lei~et al.}

\institute{University of Chinese Academy of Sciences \and
University of Amsterdam \and
United Imaging Healthcare Co., Ltd.}

\maketitle

\input{sections/0_abstract}
\input{sections/1_introduction}
\input{sections/2_related}

\input{sections/3_preliminary}
\input{sections/4_methods}
\input{sections/5_analysis}
\input{sections/6_experiments}

\input{sections/7_conclusion}

%
%
\bibliographystyle{splncs04}
\bibliography{ref}

\input{sections/8_appendix}

\end{document}

%% file: sections/0_abstract.tex
\begin{abstract}
Large Vision-Language Models (LVLMs) excel at multimodal tasks but remain prone to object hallucinations.
Prior training-free remedies often \emph{uniformly} strengthen visual signals, which may also amplify irrelevant regions and introduce spurious evidence, harming fluency.
We propose \textbf{Context-aware Attention Intervention (CAI)}, a training-free inference-time mechanism that enforces a \emph{see only when needed} principle via \emph{two-axis selectivity}: \emph{where} to look and \emph{when} to intervene.
At each decoding step, CAI derives token-specific visual relevance from early-layer representations to localize semantically aligned regions, and applies a conservative, entropy- and depth-gated attention tilt only for uncertainty-spiking tokens in deeper layers where visual grounding degrades, leaving confident tokens and irrelevant regions largely unchanged.
This targeted intervention strengthens visual grounding while preserving linguistic fluency, and it yields consistent improvements \emph{even without} contrastive decoding, which remains optional as an auxiliary bias-suppression module.
Extensive experiments across multiple LVLM backbones and benchmarks show that CAI achieves state-of-the-art hallucination mitigation, and our analysis characterizes CAI as a KL-minimal attention reweighting with bounded interference under inactive gates or small tilts. Code is available at \url{https://github.com/Iris1946/CAI}.
\keywords{Vision-Language Models \and Hallucination Mitigation \and Attention Intervention}
\end{abstract}

%% file: sections/1_introduction.tex
\section{Introduction}
\label{sec:introduction}

Large Vision-Language Models (LVLMs)~\cite{liu2023visual,dai2023instructblip,bai2023qwen,zhu2023minigpt,ye2023mplug} have achieved remarkable performance on multimodal tasks such as image captioning~\cite{li2023blip}, visual question answering~\cite{liu2023visual,dai2023instructblip}, and multimodal reasoning~\cite{huang2025vision,liu2025visual,zhou2025r1,tan2025reason,shen2025vlm}.
Despite these advances, LVLMs remain prone to hallucinations~\cite{li2023evaluating,zhou2023analyzing,liu2024survey,bai2024hallucination}, producing outputs that are linguistically plausible yet ungrounded in the visual input.
This failure mode critically undermines reliability in real-world applications where correctness and visual grounding are essential.

Prior studies~\cite{rohrbach2018object,li2023evaluating,zhou2023analyzing} attribute hallucinations to statistical biases in large-scale training data, including frequent objects and co-occurrence patterns.
Hallucinations also arise from model-intrinsic factors, particularly the over-reliance on language priors inherited from a pretrained language model~\cite{rohrbach2018object,wu2022overcoming,lee2023volcano,guan2024hallusionbench,leng2024mitigating}.
During autoregressive decoding, early mismatches between text and vision can be amplified, leading the model to commit to visually unsupported content.

\begin{figure}[t]
    \centering
    \begin{subfigure}[b]{0.3\textwidth}
        \centering
        \includegraphics[width=0.9\textwidth]{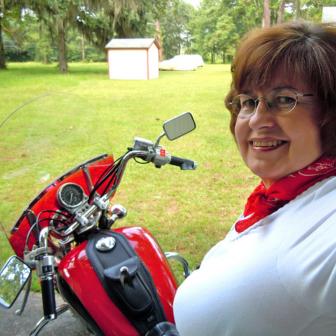}
        \caption{}
    \end{subfigure}
    \hfill
    \begin{subfigure}[b]{0.3\textwidth}
        \centering
        \includegraphics[width=0.9\textwidth]{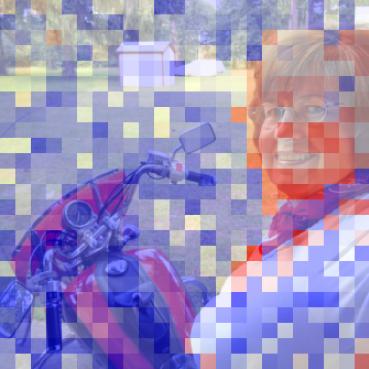}
        \caption{}
    \end{subfigure}
    \hfill
    \begin{subfigure}[b]{0.3\textwidth}
        \centering
        \includegraphics[width=0.9\textwidth]{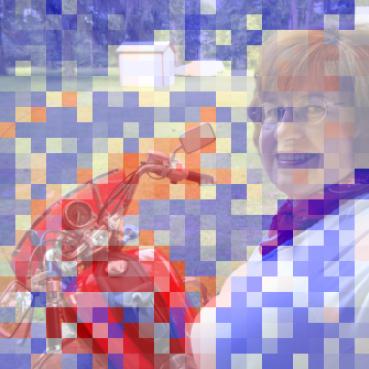}
        \caption{}
    \end{subfigure}
    \caption{\textbf{Visualization of token-image similarity.} Regions highlighted in \colorbox{red!30}{red} indicate higher relevance between generated tokens and visual content. Given visual input (a) and the query \textit{“Please describe the image in detail”}, region (b) is most associated when generating \textit{“woman”}, whereas region (c) is most relevant when generating \textit{“motorcycle”}.}
    \label{fig:vis}
\end{figure}

\begin{figure}[t]
    \centering
    \begin{subfigure}[c]{0.31\textwidth}
        \centering
        \includegraphics[width=0.9\textwidth]{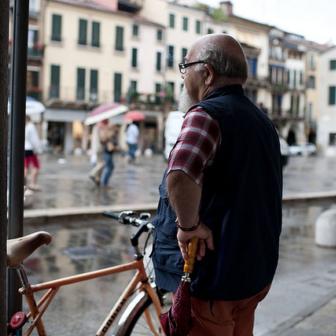}
        \caption{}
    \end{subfigure}
    \hfill
    \begin{subfigure}[c]{0.65\textwidth}
        \centering
        \includegraphics[width=0.9\textwidth]{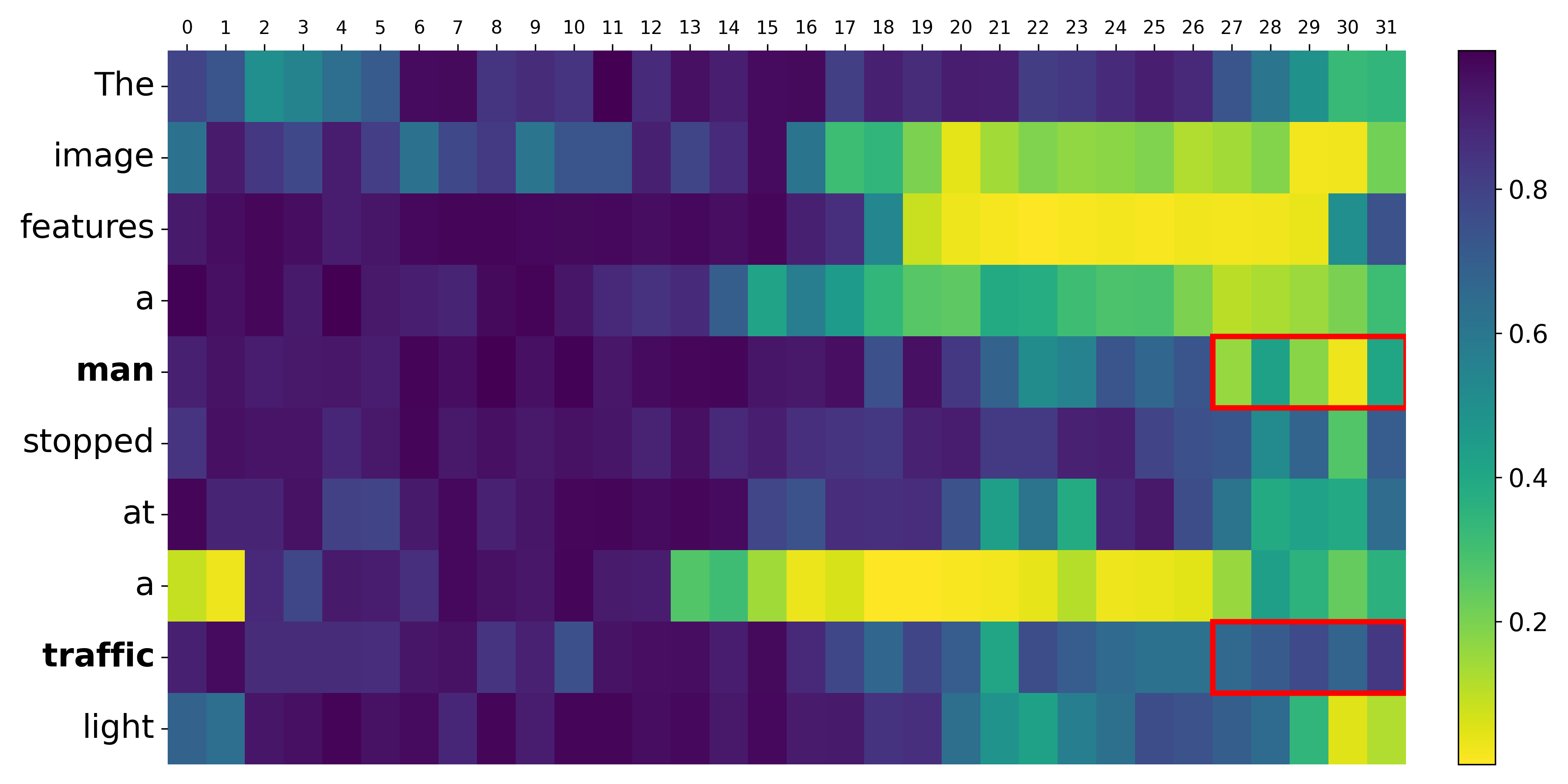}
        \caption{}
    \end{subfigure}
    \caption{Given visual input (a) and the query \textit{“Please describe the image in detail”}, (b) shows \textbf{the evolution of token entropy across decoding layers}. In deeper layers, hallucination-prone tokens (e.g., \textit{“traffic light”}) exhibit markedly higher entropy than grounded tokens (e.g., \textit{“man”}), whereas tokens dominated by language priors (e.g., \textit{“The”, “a”, “at”}) remain low-entropy.}
     \label{fig:entropy}
\end{figure}

Existing strategies for mitigating hallucinations broadly fall into training-based and training-free approaches.
Training-based methods rely on curated datasets~\cite{liu2023mitigating,yue2024less,yu2024hallucidoctor} for fine-tuning~\cite{chen2023mitigating,jiang2024hallucination,yue2024less} or reinforcement learning~\cite{sun2023aligning,yu2024rlhf,zhao2023beyond}, aiming to correct biases via additional supervision.
In contrast, training-free methods~\cite{leng2024mitigating,favero2024multi,liu2024paying,chen2025ict,an2025mitigating,liu2025reducing,zou2025memvr,wan2025only} operate purely at inference time to re-balance vision and language signals, offering a practical remedy without retraining cost.

A common design in training-free methods is to \emph{uniformly} boost visual signals, e.g., increasing attention to image tokens across decoding steps and tokens.
While effective in some cases, uniform amplification can also elevate irrelevant regions and spurious correlations, introducing new errors and degrading fluency.
This motivates a \emph{non-interference} requirement for inference-time intervention: the model does not need to ``see more'' everywhere; it should \emph{see only when needed}, strengthening visual grounding only for tokens at risk while leaving confident tokens and irrelevant regions essentially unchanged.

Our analysis suggests that such selective intervention should follow two regularities.
\emph{First}, visual relevance is \emph{token-specific}: different words should attend to different regions (Fig.~\ref{fig:vis}).
In the same image and prompt, ``woman'' and ``motorcycle'' correspond to distinct regions, indicating that strengthening vision must be spatially targeted rather than global.
\emph{Second}, hallucination risk is \emph{depth- and uncertainty-dependent}: in deeper decoding layers, hallucination-prone tokens (e.g., ``traffic light'') exhibit substantially higher predictive entropy than grounded content tokens (e.g., ``man''), whereas function words dominated by language priors (e.g., ``The'', ``a'', ``at'') remain low-entropy (Fig.~\ref{fig:entropy}).
This makes token entropy a natural \emph{risk proxy} for triggering intervention and, equally importantly, a mechanism to explicitly control \emph{when not to intervene}.
Together, these observations indicate that an effective remedy must be selective along two axes:
(i) \emph{where} to look---localize regions by token-specific relevance; and
(ii) \emph{when} to intervene---activate intervention only under high uncertainty and in deeper layers where visual grounding tends to degrade.

Based on these principles, we propose \textbf{C}ontext-aware \textbf{A}ttention \textbf{I}ntervention (\textbf{CAI}), a \emph{training-free} mechanism that performs \emph{minimal} and \emph{token-conditioned} attention intervention at inference time to enforce ``see only when needed''.
At each decoding step, \ours derives token-specific visual relevance from early-layer representations to locate semantically aligned regions, and then \emph{conditionally} applies an attention tilt toward these regions only for uncertainty-spiking tokens in deeper layers, leaving low-entropy tokens and shallow layers largely untouched.
As an optional add-on, we also report results with contrastive decoding following PAI~\cite{liu2024paying} to further suppress text-only biases; importantly, CAI itself yields consistent improvements without relying on decoding tricks.
Empirical evaluations on LLaVA-1.5~\cite{liu2024improved}, InstructBLIP~\cite{dai2023instructblip}, and Qwen-VL~\cite{bai2023qwen} show that \ours consistently reduces hallucinations and improves performance across benchmarks.

Beyond empirical gains, we provide a principled characterization of CAI to address \emph{why} and \emph{when} such selective intervention is beneficial.
Specifically, CAI can be viewed as a KL-minimal attention reweighting that increases token--image relevance with bounded perturbation; for high-uncertainty tokens, small tilts reduce the negative log-likelihood under an alignment condition between the induced visual shift and the local descent direction; and depth/entropy gating yields bounded (often negligible) interference when the gate is inactive.

\noindent\textbf{Contributions.}
(1) We introduce CAI, a training-free \emph{two-axis selective} attention intervention that decides \emph{where} to look (token-specific visual relevance) and \emph{when} to intervene (uncertainty- and depth-gated), enforcing a non-interference ``see only when needed'' principle.
(2) We present an early-to-late grounding transfer mechanism: early-layer relevance localizes aligned regions, while intervention is applied only in deeper layers where visual grounding degrades, avoiding uniform amplification and preserving fluency.
(3) We provide a principled theoretical characterization of CAI as a KL-minimal attention reweighting, together with conditions under which small tilts improve likelihood for high-uncertainty tokens and guarantees of bounded interference when the gate is inactive.

%% file: sections/2_related.tex
\section{Related Work}
\label{sec:related_work}


\noindent\textbf{Large Vision-Language Models (LVLMs).}
In recent LVLMs~\cite{liu2023visual,dai2023instructblip,bai2023qwen,zhu2023minigpt,ye2023mplug}, vision-language integration enables Large Language Models (LLMs)~\cite{brown2020language,ouyang2022training,touvron2023llama,chowdhery2023palm,bai2023qwen} to extend their reasoning beyond text by incorporating visual information. Images are first processed by a vision encoder into embeddings that are aligned with the LLM’s textual space~\cite{li2019visualbert,sun2019videobert,li2023blip}, allowing the model to interpret visual content~\cite{li2023blip}, answer questions~\cite{liu2023visual,dai2023instructblip}, and generate multimodal outputs~\cite{huang2025vision,liu2025visual,zhou2025r1,tan2025reason,shen2025vlm}. This integration effectively empowers LLMs to perform tasks that require understanding both language and vision, bridging the gap between seeing and reasoning.
Despite their capabilities, LVLMs often suffer from object hallucination~\cite{li2023evaluating,zhou2023analyzing,liu2024survey,bai2024hallucination}, describing objects that are not present in the image. Such errors reduce reliability, particularly in tasks requiring precise visual understanding. Rather than changing architectures or retraining on curated data, we act \textit{at inference time}. 
Our approach selectively reinforces attention to token-relevant visual evidence, improving fidelity while remaining training-free and plug-and-play across LVLM backbones.

\noindent\textbf{Mitigating hallucinations in LVLMs.}
Recent LVLM research has highlighted object hallucination as a persistent challenge. Hallucinations arise from both statistical biases in large-scale training data, such as frequently occurring objects and common object co-occurrences~\cite{rohrbach2018object,li2023evaluating,zhou2023analyzing}, and model-intrinsic factors, notably the reliance on language priors from pretrained language models~\cite{rohrbach2018object,wu2022overcoming,lee2023volcano,guan2024hallusionbench,leng2024mitigating}. LVLM outputs often align with language priors despite contradicting visual evidence. Strategies to mitigate hallucinations generally fall into training-based and training-free approaches. Training-based methods employ curated datasets~\cite{liu2023mitigating,yue2024less,yu2024hallucidoctor} for fine-tuning~\cite{chen2023mitigating,jiang2024hallucination,yue2024less} or reinforcement learning~\cite{sun2023aligning,yu2024rlhf,zhao2023beyond} to reduce bias-induced errors, but their high computational cost limits scalability. In contrast, training-free methods~\cite{leng2024mitigating,favero2024multi,liu2024paying,chen2025ict,an2025mitigating,liu2025reducing,zou2025memvr,wan2025only} intervene at inference time, enhancing visual grounding to counteract the model’s over-reliance on language priors, offering a more efficient alternative for improving output fidelity.
Inspired by Neo et al.~\cite{neo2024towards}, who show that visual information is localized near object tokens, we propose a \emph{token-level}, \emph{depth- and uncertainty-gated} method: it identifies token-image relevance and activates only when predictive entropy spikes in deeper layers. 
This design strengthens grounding precisely where needed and complements contrastive decoding to counter language-prior bias.

%% file: sections/3_preliminary.tex
\section{Preliminary}
\label{sec:preliminary}
LVLMs generalize Large Language Models (LLMs) to enable joint reasoning over textual and visual modalities. A vision encoder extracts visual features $\mathbf{v} = \left[v_1, \dots, v_{N_v}\right]$ from an input image, while a language model encodes textual input into query tokens $\mathbf{x} = \left[x_1, \dots, x_{N_x}\right]$. These modalities are integrated through mechanisms such as multilayer perceptrons~\cite{liu2024improved} or Q-Formers~\cite{dai2023instructblip}, generating a compact representation that conditions the language model. This fused representation facilitates autoregressive generation, formalized as:
\begin{equation}
        y_t\sim p\left(y_t|\mathbf{v},\mathbf{x},\mathbf{y}_{<t}\right)=\mathcal{S}\left(f_\theta\left(y_t|\mathbf{v},\mathbf{x},\mathbf{y}_{<t}\right)\right).
\label{eq:autoregressive}
\end{equation}
where $y_t$ denotes the token generated at step $t$, $\mathbf{y}_{<t}$ represents the preceding token sequence, and $\mathcal{S}$ is the softmax operator over the vocabulary. Here, $f_\theta$ corresponds to the language model parameterized by $\theta$. The model is implemented as a stack of transformer blocks, each comprising multi-head self-attention ($\mathrm{MHA}$) and a feed-forward network ($\mathrm{FFN}$). The attention operation for head $n$ is:
\begin{equation}
    \mathrm{Attn}_n(h)=\mathcal{S}\left(\mathbf{A}_n\right)V_n, \quad \mathbf{A}_n=\frac{Q_nK_n^\top}{\sqrt{d_k}}.
\label{eq:self_attn}
\end{equation}
where $Q_n,K_n,V_n\in\mathbb{R}^{N\times d_k}$ are the query, key, value projections of the hidden state $h$, and $d_k$ is the key dimensionality, $N=N_x+N_v$ represents the total number of multimodal tokens. The attention weights $\textbf{A}_n\in\mathbb{R}^{N\times N}$ capture token-to-token dependencies, facilitating contextual feature mixing. 
The outputs of all $H$ heads are concatenated and projected through an output matrix $W_o$:
\begin{equation}
    \mathrm{MHA}(h)=\textrm{Concat}\left(\mathrm{Attn}_1\left(h\right),\cdots,\mathrm{Attn}_H\left(h\right)\right)\cdot W_o.
\label{eq:multi_head}
\end{equation}
Finally, each transformer block applies an $\mathrm{FFN}$ to the $\mathrm{MHA}$ output, introducing nonlinear transformations that enhance contextual embeddings.

%% file: sections/4_methods.tex
\section{Method}
\label{sec:method}

\begin{figure}[t]
    \centering
    \includegraphics[width=0.95\linewidth]{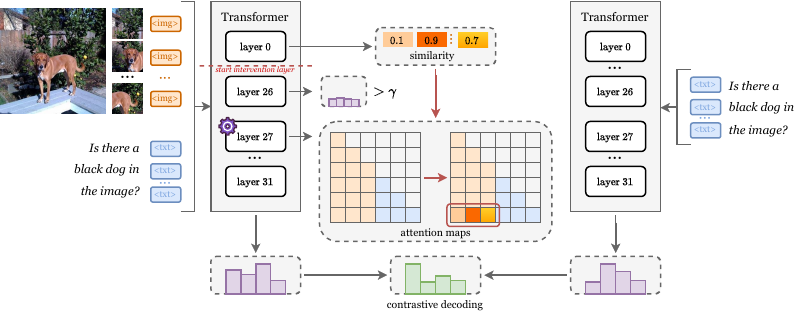}
    \caption{
    \textbf{Overview of CAI.}
    At each decoding step, the lowest layer evaluates the similarity between the generated token and each patch token~(Eq.~\ref{eq:similarity}), yielding a visual grounding signal.
    Attention interventions are applied to deep decoding layers with high entropy~(Eq.~\ref{eq:condition}), redirecting attention toward patch tokens in proportion to their similarity~(Eq.~\ref{eq:intervention}).
    Integrated with contrastive decoding~(Eq.~\ref{eq:decoding}), CAI mitigates reliance on language priors and reinforces visual grounding to reduce hallucinations.
    }
    \label{fig:cai}
\end{figure}

In this section, we present \textbf{Context-aware Attention Intervention (CAI)}, a training-free approach to mitigate hallucinations at inference time.
Previous training-free methods~\cite{leng2024mitigating,favero2024multi,liu2024paying,chen2025ict,an2025mitigating,liu2025reducing,wan2025only} typically intervene \emph{uniformly} across visual tokens, which may introduce noise and spurious correlations from irrelevant regions.
In contrast, \ours enforces a \emph{``see only when needed''} behavior by combining three components:
(i) a token-image similarity module that locates semantically relevant regions for each decoding token;
(ii) an entropy- and depth-gated attention intervention that selectively amplifies attention to those regions only when hallucination risk is high; and
(iii) an optional contrastive decoding step that further suppresses language-only continuations.
The overall pipeline of \ours is illustrated in Fig.~\ref{fig:cai}, and summarized in Algorithm~\ref{alg}.

\subsection{Token-Aware Visual Grounding}
\subsubsection{Token-image similarity.}
At step $t$, \ours measures the similarity between the current text token and the visual patch tokens. Specifically, we extract the hidden state of the last token, $h^0_t\in\mathbb{R}^{d_h}$, from the $0$-th layer, where representations remain relatively disentangled from higher-order semantic abstractions (as shown in Section~\ref{sec:experiments}).
We then compute the inner product between this $L2$-normalized hidden state and the $N_v$ visual patch features, $\mathbf{v}\in\mathbb{R}^{N_v\times d_h}$, to derive the grounding weights:

\begin{equation}
    \mathbf{w}^i_t=\sigma\left(\hat{\mathbf{v}}_i^\top\cdot \hat{h}^0_t\right),\quad \textrm{where}\quad\hat{h}_t^0=\frac{h_t^0}{\|h_t^0\|_2},\ \hat{\mathbf{v}}_i=\frac{\mathbf{v}_i}{\|\mathbf{v}_i\|_2}.
\label{eq:similarity}
\end{equation}
Here, $\sigma(\cdot)$ denotes the sigmoid function, which constrains the similarity scores to $(0,1)$, yielding $\mathbf{w}_t\in\mathbb{R}^{N_v}$.
These weights capture fine-grained text-vision alignment and provide a stable grounding signal to guide attention in deeper layers, enhancing multimodal consistency.


\subsubsection{Similarity-guided attention intervention.}
Based on the similarity $\mathbf{w}_t$, \ours performs an attention intervention on the image regions indexed by $[i_s,i_e)$ during the generation of the $t$-th token:
\begin{equation}
    \mathbf{A}_{t,i_s:i_e}
    =
    \mathbf{A}_{t,i_s:i_e}
    +
    \left|\mathbf{A}_{t,i_s:i_e}\right|\odot\mathbf{w}_t,
\label{eq:intervention}
\end{equation}
where $\mathbf{A}\in\mathbb{R}^{H\times N\times N}$ denotes the multi-head attention.
The slice $\mathbf{A}_{t,i_s:i_e}\in\mathbb{R}^{H\times N_v}$, with $N_v=i_e-i_s$, represents the attention weights from the $t$-th token to $N_v$ image tokens across $H$ heads.
The operator $\odot$ indicates element-wise multiplication.
By scaling the attention magnitude $\left|\mathbf{A}_{t,i_s:i_e}\right|$ with the similarity $\mathbf{w}_t$, the intervention amplifies attention proportionally to semantic relevance.
This mechanism reinforces context-aware visual grounding by directing the model’s focus toward the image regions that are most pertinent to the token under generation.

\subsection{Entropy- and Depth-Gated Intervention}
The indiscriminate application of interventions across all tokens and layers risks degrading multimodal representations that are already adequately grounded. To mitigate this, \ours employs a dual-gating mechanism:
\begin{equation}
    -\sum p(h_t^l)\log p(h_t^l)>\gamma,\quad \textrm{subject to}\quad l>l_s.
\label{eq:condition}
\end{equation}
First, interventions are strictly confined to layers deeper than $l_s$, where visual signals typically degrades. Second, \ours selectively targets tokens with high hallucination risk by evaluating their predictive uncertainty. Specifically, we project the layer-normalized $h_t^l$ through the language modeling head ($\textrm{LMHead}$) to compute the predictive distribution: $p(h_t^l)=\textrm{Softmax}\left(\textrm{LMHead}\left(\textrm{LayerNorm}\left(h_t^l\right)\right)\right)$. Interventions are triggered only if the entropy of this distribution exceeds $\gamma$.
By conditioning on both layer depth and entropy, \ours focuses reinforcement on hallucination-prone tokens, ensuring that interventions are precise and effective.
We further analyze this entropy- and depth-gated design in Section~\ref{sec:theory}, showing that CAI implements a KL-minimal attention tilt and yields provable gains for high-entropy tokens under small tilts.

\begin{algorithm}[!t]
\renewcommand{\algorithmicrequire}{\textbf{Input:}}
\renewcommand{\algorithmicensure}{\textbf{Output:}}
\caption{Context-aware Attention Intervention (CAI)}
\label{alg}
\begin{algorithmic}[1]
\REQUIRE Transformer layers $L$, query embedding $\mathbf{x}$, image feature $\mathbf{v}$, image start token $i_s$, image end token $i_e$, start intervention layer $l_s$, entropy threshold $\gamma$, decoding coefficient $\lambda$.
\ENSURE Response token $y_t$ at decoding step $t$.
\FOR{$l\in[0,L)$}
    \IF{$l=0$}
        \STATE $\mathbf{w}_t\leftarrow\sigma(\mathbf{v}\cdot h_t^l)$.\quad \textcolor{eccvblue}{$\%$ token-image similarity}
    \ENDIF
    \STATE \textcolor{eccvblue}{$\%$ conditional attention intervention}
    \IF{$l\ge l_s$ \textbf{and} $-\sum p(h_t^{l-1})\log p(h_t^{l-1})>\gamma$}
        \STATE $\mathbf{A}_{t,i_s:i_e}^l\leftarrow\mathbf{A}_{t,i_s:i_e}^l+\left|\mathbf{A}_{t,i_s:i_e}^l\right|\odot\mathbf{w}_t$.
    \ENDIF
    \STATE $\bar{h}_t^l\leftarrow h_t^l+\mathrm{MHA}_l(h_t^l)$.
    \STATE $h_t^{l+1}\leftarrow h_t^l+\mathrm{FFN}_l(\bar{h}_t^l)$.
\ENDFOR
\STATE $p(y_t|\mathbf{v},\mathbf{x})\leftarrow\lambda\, p(y_t|\mathbf{v},\mathbf{x})-(\lambda-1)\,p(y_t|\mathbf{x})$.\quad \textcolor{eccvblue}{$\%$ contrastive decoding}
\end{algorithmic}
\end{algorithm}

\subsection{Contrastive Decoding}
In CAI, contrastive decoding is introduced as an optional component. Following PAI~\cite{liu2024paying}, hallucinations from over-reliance on language priors are mitigated by contrasting multimodal prediction $p(y_t|\mathbf{v},\mathbf{x})$ with text-only prediction $p(y_t|\mathbf{x})$:
\begin{equation}
    \hat{p}(y_t|\mathbf{v},\mathbf{x})
    =
    \lambda\, p(y_t|\mathbf{v},\mathbf{x})
    -
    (\lambda-1)\, p(y_t|\mathbf{x}),
\label{eq:decoding}
\end{equation}
where $\lambda$ is a contrastive decoding coefficient.
This penalizes text-only hypotheses to encourage visual grounding.
In contrast to PAI~\cite{liu2024paying}, where attention intervention alone yields limited impact and contrastive decoding drives the significant gains reported in Section~\ref{sec:experiments}, our \ours achieves substantial hallucination mitigation with the standalone attention intervention mechanism (CAI$^\dagger$, i.e., $\lambda=1$), making contrastive decoding optional.

%% file: sections/5_analysis.tex
\section{Theoretical Analysis}
\label{sec:theory}

We next provide a formal analysis of \ours by viewing it as an \emph{exponential tilt} of the baseline attention.
Under this lens, CAI enjoys four desirable properties that mirror our design principles.
First, among all ways of modifying attention within a fixed KL budget, the CAI update is the \emph{KL-minimal} reweighting that most increases expected token–image similarity (Theorem~\ref{thm:kl-minimal}).
Second, for high-entropy tokens, small tilts that shift visual evidence in the NLL descent direction provably reduce the loss (Theorem~\ref{thm:entropy-gain}), matching our entropy-gated intervention.
Third, under a simple model where visual information decays with depth, later layers yield larger marginal benefits from reinforcement (Theorem~\ref{thm:depth-decay}), justifying depth-gated intervention.
Finally, for small tilts or weak similarity, the change in attention is tightly bounded in total variation (Theorem~\ref{thm:bounded-risk}), ensuring that CAI behaves as a low-risk, ``see only when needed'' modification.

\paragraph{Setup.}
Let $a\in\Delta^{N-1}$ be the attention over $N$ visual tokens at the current step and $s\in\mathbb{R}^N_{\ge0}$ the token-image similarity.
When gated by $\mathbb{I}[H_t^{(l-1)}>\gamma]\cdot\mathbb{I}[\,l\ge l_0\,]$, \ours applies a multiplicative tilt:
\[
\tilde a_i \;=\; \frac{a_i\,\exp(\lambda s_i)}{\sum_{j} a_j\,\exp(\lambda s_j)} ,\qquad \lambda\ge 0,
\]
else $\tilde a=a$.
Let $x(a)=\sum_i a_i v_i$ be the aggregated visual evidence and logits $z_y = u_y + w_y^\top x(a)$, with NLL $\mathcal{L}(a)=-\log\mathrm{softmax}(z)_{y^\star}$.

\begin{theorem}[KL-minimality of CAI tilting]
\label{thm:kl-minimal}
For any baseline $a$ and similarity $s$, the distribution
$q^\star(\lambda)\propto a\odot \exp(\lambda s)$ uniquely solves
\(
\max_{q\in\Delta^{N-1}} \ \mathbb{E}_{q}[s] \ \\ \text{s.t.}\ D_{\mathrm{KL}}(q\,\|\,a)\le \varepsilon
\)
for some $\lambda\ge 0$ meeting the KL budget.
Thus, CAI realizes the least-change reweighting that raises expected relevance.
\end{theorem}

\begin{theorem}[Entropy-gated improvement]
\label{thm:entropy-gain}
Assume the linearized logit model $z_y = u_y + w_y^\top x(a)$ and local smoothness of $\log\sum_y e^{z_y}$.
Let $H=-\sum_y p_y\log p_y$ with $p=\mathrm{softmax}(z)$.
If $H\ge H_0>0$ and the CAI direction aligns with the NLL descent, i.e.,
$\langle g(a),\,\Delta x\rangle>0$ where $g(a)=\sum_y (p_y-\mathbb{1}[y{=}y^\star])w_y$ and $\Delta x=x(\tilde a)-x(a)$,
then there exists $\lambda_0>0$ such that for all $0<\lambda\le \lambda_0$,
\(
\mathcal{L}(\tilde a) < \mathcal{L}(a).
\)
Hence high-entropy tokens are the regime where CAI is guaranteed to help for sufficiently small tilts.
\end{theorem}

\begin{theorem}[Depth advantage via visual decay]
\label{thm:depth-decay}
Suppose the visual component of hidden states evolves as $x^{(l)} = M^{(l)}x^{(l-1)}$ with $\mathbb{E}\,\rho(M^{(l)})\le \rho<1$.
Then for $l\ge l_0$, $\|x^{(l)}\|\le \rho^{\,l-l_0}\|x^{(l_0)}\|$.
Therefore the signal-to-noise ratio of visual evidence decays geometrically with depth, making depth-gated intervention ($l\!\ge\! l_0$) yield larger marginal returns.
\end{theorem}

\begin{theorem}[Non-interference under small tilts]
\label{thm:bounded-risk}
For small $\lambda$, $D_{\mathrm{KL}}(\tilde a\,\|\,a)=\tfrac{1}{2}\lambda^2\mathrm{Var}_{a}[s]+o(\lambda^2)$; by Pinsker, 
$\|\tilde a-a\|_{\mathrm{TV}}\le \sqrt{\tfrac12 D_{\mathrm{KL}}(\tilde a\,\|\,a)}$.
Thus, when the gate is off or $s$ is weak, CAI perturbs attention only negligibly.
\end{theorem}

\noindent\textit{Proofs of Theorems~\ref{thm:kl-minimal}--\ref{thm:bounded-risk} are provided in Appendix.}

%% file: sections/6_experiments.tex
\section{Experiments}
\label{sec:experiments}

\begin{table}[!b]
\caption{\textbf{CHAIR evaluation results} for LLaVA-1.5 and InstructBLIP.}
\centering
\resizebox{0.8\textwidth}{!}{%
\begin{tabular}{llcccc}
\toprule
\multirow{2}{*}{} & \multirow{2}{*}{\textbf{Method}} & \multicolumn{2}{c}{Max Token = $64$} & \multicolumn{2}{c}{Max Token = $128$} \\
\cmidrule(lr){3-4} \cmidrule(lr){5-6}
 &  & \textit{CHAIR}$_S \downarrow$ & \textit{CHAIR}$_I \downarrow$ & \textit{CHAIR}$_S \downarrow$ & \textit{CHAIR}$_I \downarrow$ \\ \midrule
\multirow{5}{*}{\textbf{LLaVA-1.5}} & Regular & 26.0 & 8.8 & 56.6 & 16.8 \\
 & VCD & 23.8 & 8.2 & 59.6 & 16.6  \\
 & PAI & 26.0 & 8.7 & 54.0 & 15.2 \\
 & ONLY & 21.0 & 7.3 & 48.4 & 14.7 \\
 & \cellcolor{lightgray!30}\textbf{\ours} & \cellcolor{lightgray!30}\textbf{17.8} & \cellcolor{lightgray!30}\textbf{6.9} & \cellcolor{lightgray!30}\textbf{39.6} & \cellcolor{lightgray!30}\textbf{12.4} \\ \midrule
\multirow{5}{*}{\textbf{InstructBLIP}} & Regular & 29.0 & 10.2 & 54.2 & 16.7 \\
 & VCD & 26.4 & 8.7 & 55.8 & \textbf{15.9} \\
 & PAI & 24.6 & 8.3 & 56.8 & 16.7 \\
 & ONLY & 25.2 & 8.8 & 55.8 & 17.4 \\
 & \cellcolor{lightgray!30}\textbf{\ours} & \cellcolor{lightgray!30}\textbf{24.2} & \cellcolor{lightgray!30}\textbf{8.2} & \cellcolor{lightgray!30}\textbf{52.8} & \cellcolor{lightgray!30}16.5 \\ \bottomrule
\end{tabular}
}
\label{tab:chair}
\end{table}

\subsection{Experiment Setup}
\noindent\textbf{Datasets.}
Our evaluation employs three benchmarks.
\textbf{POPE}~\cite{li2023evaluating} detects hallucinations via binary object-existence queries on MS-COCO~\cite{lin2014microsoft}, A-OKVQA~\cite{schwenk2022okvqa}, and GQA~\cite{hudson2019gqa}, using \textit{random}, \textit{popular}, and \textit{adversarial} sampling to assess accuracy, memorization bias, and robustness.
\textbf{CHAIR}~\cite{rohrbach2018object} evaluates hallucinations in free-form captioning by quantifying both the proportion of hallucinated object instances and the proportion of captions containing hallucinations.
\textbf{MME}~\cite{yin2024survey} delivers a comprehensive evaluation across fourteen subtasks formulated as yes-or-no queries, encompassing perceptual dimensions such as object existence, count, position, and color.

\noindent\textbf{Implementation details.}
Our evaluation is conducted on LLaVA-1.5~\cite{liu2024improved}, InstructBLIP~\cite{dai2023instructblip}, and Qwen-VL~\cite{bai2023qwen}. Baselines, including VCD~\cite{leng2024mitigating}, PAI~\cite{liu2024paying} and ONLY~\cite{wan2025only}, are employed with their default configurations to ensure a fair comparison. For our approach, we perform a grid search over the hyperparameters $l_s$, $\gamma$ and $\lambda$, while $\sigma$ is instantiated as the $\mathrm{sigmoid}$ function. All experiments are executed on a single 80GB NVIDIA A800 GPU.

\begin{table}[!t]
\caption{\textbf{POPE evaluation results} for LLaVA-1.5, InstructBLIP, and Qwen-VL.}
\centering
\resizebox{0.75\textwidth}{!}{%
\begin{tabular}{p{7mm}p{19mm}p{14mm}p{12mm}p{12mm}p{12mm}p{12mm}p{12mm}p{12mm}}
\toprule
\multicolumn{2}{c}{\multirow{2}{*}{\textbf{Dataset}}} & \multirow{2}{*}{\textbf{Method}} & \multicolumn{2}{c}{\textbf{LLaVA-1.5}} & \multicolumn{2}{c}{\textbf{InstructBLIP}} & \multicolumn{2}{c}{\textbf{Qwen-VL}} \\
\cmidrule(lr){4-5} \cmidrule(lr){6-7} \cmidrule(lr){8-9} 
\multicolumn{2}{l}{} &  & \textit{Acc} $\uparrow$ & \textit{F1} $\uparrow$ & \textit{Acc} $\uparrow$ & \textit{F1} $\uparrow$ & \textit{Acc} $\uparrow$ & \textit{F1} $\uparrow$ \\ \midrule
\multirow{15}{*}{\rotatebox{90}{\textbf{MS-COCO}}} & \multirow{5}{*}{Random} & Regular & 85.46 & 86.07 & 82.82 & 83.56 & 85.09 & 83.46 \\
 &  & VCD & 85.74 & 86.70 & 85.53 & 85.32 & 89.59 & 89.18 \\
 &  & PAI & 86.67 & 87.26 & 85.43 & 83.60 & 85.74 & 84.20 \\
 &  & ONLY & 88.38 & 88.41 & 87.32 & 87.77 & 88.97 & 88.37 \\
 &  & \cellcolor{lightgray!30}\textbf{\ours} & \cellcolor{lightgray!30}\textbf{89.24} & \cellcolor{lightgray!30}\textbf{89.14} & \cellcolor{lightgray!30}\textbf{89.55} & \cellcolor{lightgray!30}\textbf{89.35} & \cellcolor{lightgray!30}\textbf{89.90} & \cellcolor{lightgray!30}\textbf{89.49} \\ \cmidrule{2-9}
 & \multirow{5}{*}{Popular} & Regular & 81.20 & 82.22 & 75.77 & 77.69 & 84.13 & 81.97 \\
 &  & VCD & 81.93 & 83.36 & 80.97 & 81.15 & 87.57 & 87.02 \\
 &  & PAI & 82.77 & 83.60 & 77.13 & 79.14 & 85.10 & 83.09 \\
 &  & ONLY & 85.00 & 85.23 & 77.87 & 79.93 & 87.23 & 86.31 \\
 &  & \cellcolor{lightgray!30}\textbf{\ours} & \cellcolor{lightgray!30}\textbf{86.03} & \cellcolor{lightgray!30}\textbf{85.53} & \cellcolor{lightgray!30}\textbf{84.30} & \cellcolor{lightgray!30}\textbf{83.65} & \cellcolor{lightgray!30}\textbf{88.30} & \cellcolor{lightgray!30}\textbf{87.70} \\ \cmidrule{2-9}
 & \multirow{5}{*}{Adversarial} & Regular & 75.87 & 78.27 & 74.23 & 76.61 & 81.50 & 79.51 \\
 &  & VCD & 76.90 & 79.62 & 79.17 & 79.99 & 84.20 & 84.07 \\
 &  & PAI & 76.83 & 79.14 & 74.87 & 77.53 & 83.33 & 81.64 \\
 &  & ONLY & 79.93 & 81.18 & 75.57 & 78.29 & 84.23 & 83.67 \\
 &  & \cellcolor{lightgray!30}\textbf{\ours} & \cellcolor{lightgray!30}\textbf{82.77} & \cellcolor{lightgray!30}\textbf{82.73} & \cellcolor{lightgray!30}\textbf{81.83} & \cellcolor{lightgray!30}\textbf{81.56} & \cellcolor{lightgray!30}\textbf{84.97} & \cellcolor{lightgray!30}\textbf{84.69} \\ \midrule
\multirow{15}{*}{\rotatebox{90}{\textbf{A-OKVQA}}} & \multirow{5}{*}{Random} & Regular & 82.07 & 83.31 & 81.53 & 82.86 & 86.53 & 85.19 \\
 &  & VCD & 82.17 & 84.14 & 85.27 & 85.72 & 89.77 & 89.50 \\
 &  & PAI & 83.33 & 84.85 & 83.10 & 84.46 & 87.03 & 85.80 \\
 &  & ONLY & 86.03 & 86.93 & 85.67 & 86.70 & 89.23 & 88.78 \\
 &  & \cellcolor{lightgray!30}\textbf{\ours} & \cellcolor{lightgray!30}\textbf{89.00} & \cellcolor{lightgray!30}\textbf{89.03} & \cellcolor{lightgray!30}\textbf{89.87} & \cellcolor{lightgray!30}\textbf{89.73} & \cellcolor{lightgray!30}\textbf{90.07} & \cellcolor{lightgray!30}\textbf{89.91} \\ \cmidrule{2-9}
 & \multirow{5}{*}{Popular} & Regular & 75.30 & 78.99 & 74.80 & 77.98 & 86.00 & 84.78 \\
 &  & VCD & 77.20 & 80.54 & 78.97 & 80.76 & 89.53 & 89.34 \\
 &  & PAI & 76.37 & 79.79 & 76.47 & 79.61 & 87.37 & 86.17 \\
 &  & ONLY & 79.40 & 81.86 & 76.73 & 80.07 & 88.83 & 88.46 \\
 &  & \cellcolor{lightgray!30}\textbf{\ours} & \cellcolor{lightgray!30}\textbf{84.60} & \cellcolor{lightgray!30}\textbf{85.23} & \cellcolor{lightgray!30}\textbf{84.87} & \cellcolor{lightgray!30}\textbf{85.40} & \cellcolor{lightgray!30}\textbf{89.53} & \cellcolor{lightgray!30}\textbf{89.43} \\ \cmidrule{2-9}
 & \multirow{5}{*}{Adversarial} & Regular & 67.07 & 73.70 & 68.33 & 73.89 & 81.03 & 80.36 \\
 &  & VCD & 68.30 & 74.86 & 73.33 & 77.00 & 82.13 & 82.92 \\
 &  & PAI & 67.60 & 74.23 & 68.67 & 74.55 & 82.10 & 81.58 \\
 &  & ONLY & 70.47 & 75.86 & 68.77 & 74.95 & 82.20 & 82.74 \\
 &  & \cellcolor{lightgray!30}\textbf{\ours} & \cellcolor{lightgray!30}\textbf{77.40} & \cellcolor{lightgray!30}\textbf{79.79} & \cellcolor{lightgray!30}\textbf{75.23} & \cellcolor{lightgray!30}\textbf{78.04} & \cellcolor{lightgray!30}\textbf{82.60} & \cellcolor{lightgray!30}\textbf{83.20} \\ \midrule
\multirow{15}{*}{\rotatebox{90}{\textbf{GQA}}} & \multirow{5}{*}{Random} & Regular & 82.03 & 83.86 & 80.17 & 81.55 & 83.83 & 82.60 \\
 &  & VCD & 81.70 & 83.99 & 83.37 & 83.78 & 88.33 & 88.26 \\
 &  & PAI & 83.37 & 85.03 & 81.23 & 82.67 & 85.90 & 84.77 \\
 &  & ONLY & 86.20 & 87.20 & 83.87 & 85.02 & 88.13 & 87.54 \\
 &  & \cellcolor{lightgray!30}\textbf{\ours} & \cellcolor{lightgray!30}\textbf{89.10} & \cellcolor{lightgray!30}\textbf{89.15} & \cellcolor{lightgray!30}\textbf{88.10} & \cellcolor{lightgray!30}\textbf{87.84} & \cellcolor{lightgray!30}\textbf{90.13} & \cellcolor{lightgray!30}\textbf{89.87} \\ \cmidrule{2-9}
 & \multirow{5}{*}{Popular} & Regular & 71.93 & 76.88 & 72.27 & 75.97 & 80.77 & 80.15 \\
 &  & VCD & 74.37 & 78.97 & 77.57 & 79.40 & 84.33 & 84.82 \\
 &  & PAI & 72.73 & 77.60 & 73.77 & 77.34 & 82.67 & 81.89 \\
 &  & ONLY & 75.93 & 79.62 & 74.80 & 78.41 & 82.80 & 82.72 \\
 &  & \cellcolor{lightgray!30}\textbf{\ours} & \cellcolor{lightgray!30}\textbf{83.43} & \cellcolor{lightgray!30}\textbf{84.39} & \cellcolor{lightgray!30}\textbf{81.20} & \cellcolor{lightgray!30}\textbf{82.05} & \cellcolor{lightgray!30}\textbf{84.97} & \cellcolor{lightgray!30}\textbf{84.95} \\ \cmidrule{2-9}
 & \multirow{5}{*}{Adversarial} & Regular & 67.93 & 74.35 & 68.33 & 73.25 & 79.20 & 78.96 \\
 &  & VCD & 68.63 & 75.41 & 73.40 & 76.38 & 81.70 & 82.74 \\
 &  & PAI & 69.00 & 75.34 & 69.10 & 74.19 & 80.87 & 80.41 \\
 &  & ONLY & 71.50 & 76.75 & 69.37 & 74.92 & 81.43 & 81.67 \\
 &  & \cellcolor{lightgray!30}\textbf{\ours} & \cellcolor{lightgray!30}\textbf{78.33} & \cellcolor{lightgray!30}\textbf{80.46} & \cellcolor{lightgray!30}\textbf{76.00} & \cellcolor{lightgray!30}\textbf{77.65} & \cellcolor{lightgray!30}\textbf{83.27} & \cellcolor{lightgray!30}\textbf{83.51} \\ 
 \bottomrule
\end{tabular}
}
\label{tab:pope}
\end{table}

\begin{table}[!b]
\caption{\textbf{MME evaluation results} for LLaVA-1.5, InstructBLIP, and Qwen-VL.}
\centering
\resizebox{0.8\textwidth}{!}{%
\begin{tabular}{llccccc}
\toprule
 & \multirow{2}{*}{\textbf{Method}} & \multicolumn{2}{c}{Object-level} & \multicolumn{2}{c}{Attribute-level} & \multirow{2}{*}{\textit{ Score} $\uparrow$}\\
 \cmidrule(lr){3-4} \cmidrule(lr){5-6}
 &  & \textit{Existence} $\uparrow$ & \textit{Count} $\uparrow$ & \textit{Position} $\uparrow$ & \textit{Color} $\uparrow$ &  \\
 \midrule
\multirow{5}{*}{\textbf{LLaVA-1.5}} & Regular & 185.00 & 126.67 & 128.33 & 148.33 & 588.33 \\
 & VCD & 180.00 & 141.67& 128.33 & 153.33 & 603.33 \\
 & PAI & 185.00 & 131.67 & 133.33 & 153.33 & 603.33 \\
 & ONLY & 190.00 & 141.67 & 133.33 & 163.33 & 628.33 \\
 & \cellcolor{lightgray!30}\textbf{\ours} & \cellcolor{lightgray!30}\textbf{190.00} & \cellcolor{lightgray!30}\textbf{143.33}& \cellcolor{lightgray!30}\textbf{148.33}& \cellcolor{lightgray!30}\textbf{178.33}& \cellcolor{lightgray!30}\textbf{660.00}\\ \midrule
\multirow{5}{*}{\textbf{InstructBLIP}} & Regular & 170.00 & 75.00 & 68.33 & 140.00 & 453.33 \\
 & VCD & 155.00 & 78.33 & 76.67 & 155.00& 465.00 \\
 & PAI & 150.00 & 88.33 & \textbf{78.33} & 150.00 & 466.67\\
 & ONLY & 170.00 & 68.33 & 63.33 & 145.00 & 446.66 \\
 & \cellcolor{lightgray!30}\textbf{\ours} & \cellcolor{lightgray!30}\textbf{175.00} & \cellcolor{lightgray!30}\textbf{100.00}& \cellcolor{lightgray!30}70.00 & \cellcolor{lightgray!30}\textbf{165.00} & \cellcolor{lightgray!30}\textbf{510.00}\\ \midrule
\multirow{5}{*}{\textbf{Qwen-VL}} & Regular & 165.00 & 135.00 & 163.33& 175.00 & 638.33 \\
 & VCD & 170.00 & 120.00 & 133.33 & 175.00 & 598.33 \\
 & PAI & 175.00 & 135.00 & 163.33 & 175.00 & 648.33 \\
 & ONLY & 180.00 & 136.67 & 153.33 & 180.00 & 650.00 \\
 & \cellcolor{lightgray!30}\textbf{\ours} & \cellcolor{lightgray!30}\textbf{180.00} & \cellcolor{lightgray!30}\textbf{146.67} & \cellcolor{lightgray!30}\textbf{163.33} & \cellcolor{lightgray!30}\textbf{185.00} & \cellcolor{lightgray!30}\textbf{675.00} \\ 
\bottomrule
\end{tabular}
}
\label{tab:mme}
\end{table}

\begin{table}[!b]
\caption{\textbf{Efficiency–performance trade-off on LLaVA-1.5.}
CAI$^\dagger$ and PAI$^\dagger$ uses attention intervention only, CAI and PAI additionally applies contrastive decoding.}
\centering
\resizebox{\textwidth}{!}{%
\begin{tabular}{lllllllclclcl}
\toprule
\textbf{Method} & \multicolumn{2}{c}{\begin{tabular}[c]{@{}c@{}}Latency $\downarrow$\\ (ms/token)\end{tabular}} & \multicolumn{2}{c}{\begin{tabular}[c]{@{}c@{}}Throughput $\uparrow$\\ (token/ms)\end{tabular}} & \multicolumn{2}{c}{\begin{tabular}[c]{@{}c@{}}GPU Memory $\downarrow$\\ (MB)\end{tabular}} & \multicolumn{2}{c}{\textit{POPE $\uparrow$}} & \multicolumn{2}{c}{\textit{CHAIR $\downarrow$}} & \multicolumn{2}{c}{\textit{MME $\uparrow$}} \\ \midrule
Regular & \multicolumn{2}{c}{89.62} & \multicolumn{2}{c}{9.41} & \multicolumn{2}{c}{14241} & \multicolumn{2}{c}{74.81} & \multicolumn{2}{c}{56.6} & \multicolumn{2}{c}{588.33} \\
VCD & 215.22 & $_{ (\times 2.40)}$ & 2.22 & $_{ (\times 0.24)}$ & 15299 & $_{ (\times 1.07)}$ & 75.89 & {\textcolor{eccvblue}{$_{(+1.08)}$}} & 59.6 & {\color[HTML]{FF0000} $_{(+3.0)}$} & 603.33 & {\textcolor{eccvblue}{$_{(+15.00)}$}} \\
ONLY & 116.51 & $_{ (\times 1.30)}$ & 6.31 & $_{ (\times 0.67)}$ & 14281 & $_{ (\times 1.00)}$ & 78.63 & {\textcolor{eccvblue}{$_{(+3.82)}$}} & 48.4 & {\textcolor{eccvblue}{$_{(-8.2)}$}} & 628.33 & {\textcolor{eccvblue}{$_{(+40.00)}$}} \\
\midrule
PAI$^{\dagger}$ & 90.97 & $_{ (\times 1.02)}$ & 9.27 & $_{ (\times 0.99)}$ & 14241 & $_{ (\times 1.00)}$ & 74.57 & {\color[HTML]{FF0000} $_{(-0.24)}$} & 54.8 & {\textcolor{eccvblue}{$_{(-1.8)}$}} & 581.66 & {\color[HTML]{FF0000} $_{(-6.67)}$} \\
\rowcolor{gray!20} \textbf{CAI}$^{\dagger}$ & \textbf{92.35} & $_{ (\times 1.03)}$ & \textbf{9.15} & $_{ (\times 0.97)}$ & \textbf{14251} & $_{ (\times 1.00)}$ & 77.13 & {\textcolor{eccvblue}{$_{(+2.32)}$}} & 43.6 & {\textcolor{eccvblue}{$_{(-13.0)}$}} & 608.33 & {\textcolor{eccvblue}{$_{(+20.00)}$}} \\
\midrule
PAI & 123.57 & $_{ (\times 1.38)}$ & 7.09 & $_{ (\times 0.75)}$ & 14281 & $_{ (\times 1.00)}$ & 75.77 & {\textcolor{eccvblue}{$_{(+0.96)}$}} & 54.0 & {\textcolor{eccvblue}{$_{(-2.6)}$}} & 603.33 & {\textcolor{eccvblue}{$_{(+15.00)}$}} \\
\rowcolor{gray!20} \textbf{\ours} & 131.33 & $_{ (\times 1.47)}$ & 6.8 & $_{ (\times 0.72)}$ & 14291 & $_{ (\times 1.00)}$ & \textbf{83.67} & {\textcolor{eccvblue}{$_{(+8.86)}$}} & \textbf{39.6} & {\textcolor{eccvblue}{$_{(-17.0)}$}} & \textbf{660.00} & {\textcolor{eccvblue}{$_{(+71.67)}$}} \\ \bottomrule
\end{tabular}
}%
\label{tab:efficiency}
\end{table}

\subsection{Results}

\noindent\textbf{Overall performance across benchmarks.}
Across POPE, CHAIR, and MME, \ours consistently improves hallucination mitigation and visual grounding over the strongest training-free baseline, and the gains hold across three representative LVLM backbones (LLaVA-1.5, InstructBLIP, and Qwen-VL), indicating that CAI is not model-specific but transfers reliably.

\noindent\textbf{Hallucination mitigation on POPE and CHAIR.}
Table~\ref{tab:pope} shows consistent improvements on POPE across all backbones, reflecting fewer visually ungrounded responses under a controlled hallucination protocol.
Table~\ref{tab:chair} further corroborates this trend on free-form captions: \ours reduces both sentence-level and instance-level hallucination rates, and the reduction remains under different maximum output lengths, suggesting that CAI mitigates error accumulation during autoregressive decoding.

\noindent\textbf{General visual reasoning on MME.}
Beyond hallucination-specific metrics, Table~\ref{tab:mme} demonstrates improved performance on MME for both object-level (existence, count) and attribute-level (position, color) reasoning.
This aligns with CAI's design---token-specific region localization strengthens fine-grained grounding, while entropy- and depth-gated intervention avoids uniform amplification of irrelevant regions, leading to more reliable object and attribute predictions.

\subsection{Discussion}
\begin{figure}[!b]
    \centering
    \begin{minipage}[t
    ]{0.38\textwidth}
        \centering
        \includegraphics[width=\textwidth]{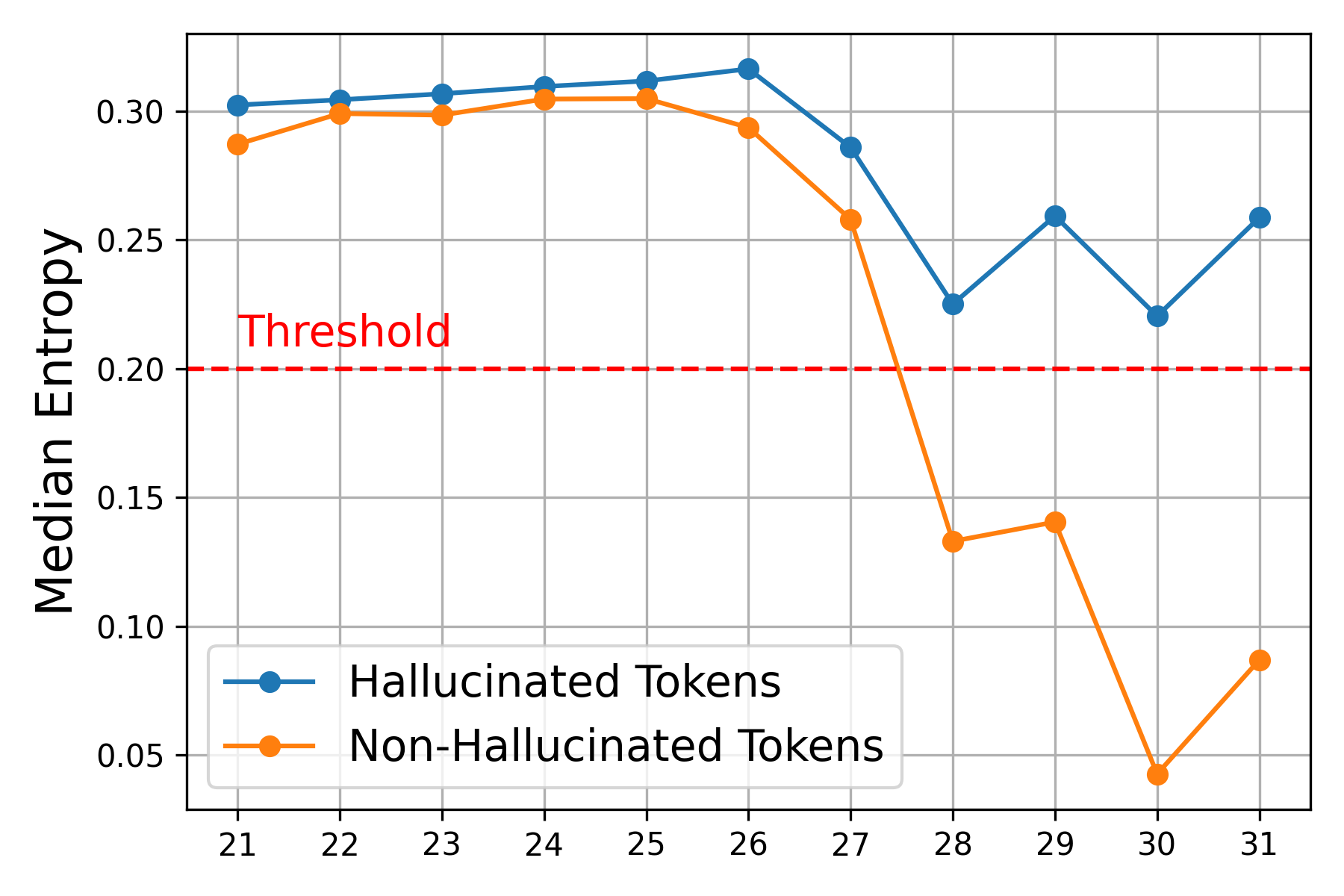}
        \caption{\textbf{Median entropy of tokens across layers.} With an intervention threshold, hallucination and non-hallucination tokens are distinctly separable in the deeper layers.}
        \label{fig:app_entropy}
    \end{minipage}
    \hfill
    \begin{minipage}[c]{0.6\textwidth}
        \centering
        \begin{subfigure}[c]{0.38\textwidth}
            \centering
            \includegraphics[width=\textwidth]{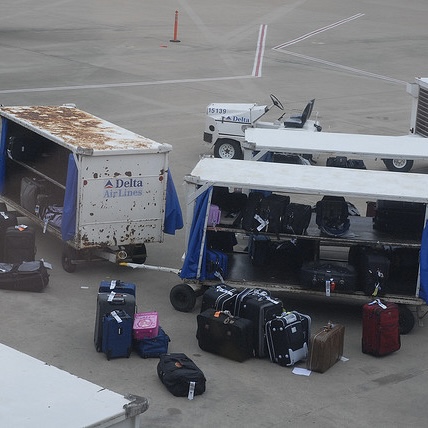}
            \caption{Image sample}
        \end{subfigure}
        \begin{subfigure}[c]{0.6\textwidth}
            \centering
            \includegraphics[width=\textwidth]{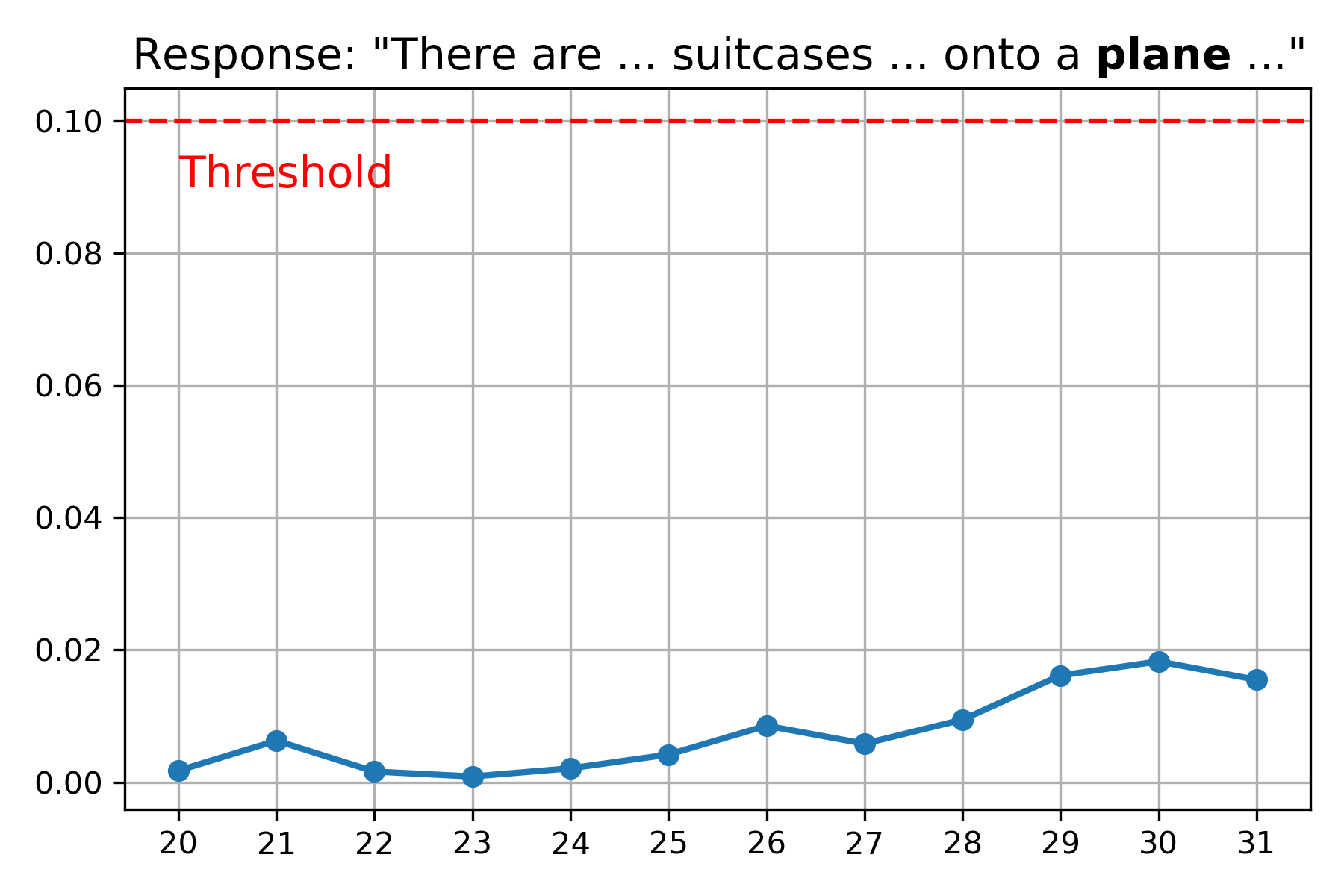}
            \caption{Layer entropy of “plane”}
        \end{subfigure}
    \caption{\textbf{Case study: failure mode of confident hallucination.} Given the prompt “Please describe this image in detail”, the model generates the hallucinated token “plane”, which is not present in (a). Note that its predictive entropy across multiple layers remains below threshold, illustrating a failure mode where confident hallucinations evade entropy-based intervention.}
    \label{fig:reb_fail}
    \end{minipage}
\end{figure}

\noindent\textbf{Analysis of contrastive decoding.}
Table~\ref{tab:efficiency} summarizes the accuracy–efficiency trade-offs. The attention-only variant (CAI$^{\dagger}$) consistently outperforms the {Regular} baseline on POPE, CHAIR, and MME while maintaining comparable latency and throughput. This demonstrates that the principal gains stem from the proposed token- and risk-aware attention intervention, rather than from decoding tricks.
In contrast, PAI$^{\dagger}$, which applies uniform amplification, yields limited or even negative improvements, supporting our hypothesis that indiscriminate boosting can introduce irrelevant evidence and impair grounding.

Adding contrastive decoding (CAI$^{\dagger}{\rightarrow}$CAI) further improves performance across benchmarks with only marginal overhead. For fairness, CAI adopts the same contrastive decoding scheme as PAI; however, the results indicate that CAI does not depend on it for its core improvements. Instead, contrastive decoding serves as a complementary module that provides incremental precision, particularly under elevated hallucination risk. Overall, the key distinction lies in where, when, and how attention is modulated—through token-specific relevance estimation, entropy- and depth-based gating, and KL-minimal reweighting—while decoding operates as a shared auxiliary component rather than the primary source of gain.

\noindent\textbf{Effect of deep-layer entropy in detecting hallucination.}
We adopt predictive entropy as a training-free, model-internal signal available at inference time and empirically validate its utility as a hallucination risk indicator. Predictive entropy is used only as a conservative trigger to avoid over-correction: high-entropy tokens concentrate hallucination risk in deeper layers (Fig.~\ref{fig:app_entropy}), so gating preserves fluency by leaving confident tokens untouched.
Fig.~\ref{fig:reb_fail} highlights a known failure mode: \textit{confident hallucination}. In this case, entropy stays below the threshold, and CAI is not triggered. This limitation is inherent to uncertainty-based gating, and it motivates using contrastive decoding as an optional safeguard against text-only hypotheses.

\noindent\textbf{Effect of $l_s$ and $\gamma$ in intervention conditions.}
We sweep the start layer $l_s\!\in\![25,30]$ and entropy threshold $\gamma\!\in\!\{0.05,0.1,0.15\}$ on A-OKVQA (random setting).
Fig.~\ref{fig:ablation_start_layer} peaks at $l_s=27$ and $\gamma=0.1$.
Intervening too \emph{early} amplifies shallow-layer noise, while intervening too \emph{late} misses error accumulation; similarly, a threshold that is too \emph{low} over-triggers on easy tokens, and too \emph{high} under-triggers on genuinely uncertain tokens.
These trends support our design: decide \emph{where} to look by token-image similarity, and \emph{when} to act by depth and predictive entropy.

\noindent\textbf{Effect of $\lambda$ in contrastive decoding.}
We grid-search the decoding coefficient $\lambda$ (Fig.~\ref{fig:ablation_lambda}); accuracy is maximized at $\lambda=3.0$ on A-OKVQA (random setting).
Small $\lambda$ under-penalizes language-prior continuations; excessively large $\lambda$ over-restricts decoding and can hurt fluency.
We use $\lambda\approx3$ when precision is prioritized, and $\lambda=1$ (i.e., no contrastive decoding, \textit{CAI$^\dagger$}) in strict latency budgets.

\begin{figure}[t]
    \centering
    \begin{minipage}[t]{0.48\textwidth}
        \centering
        \includegraphics[width=\textwidth]{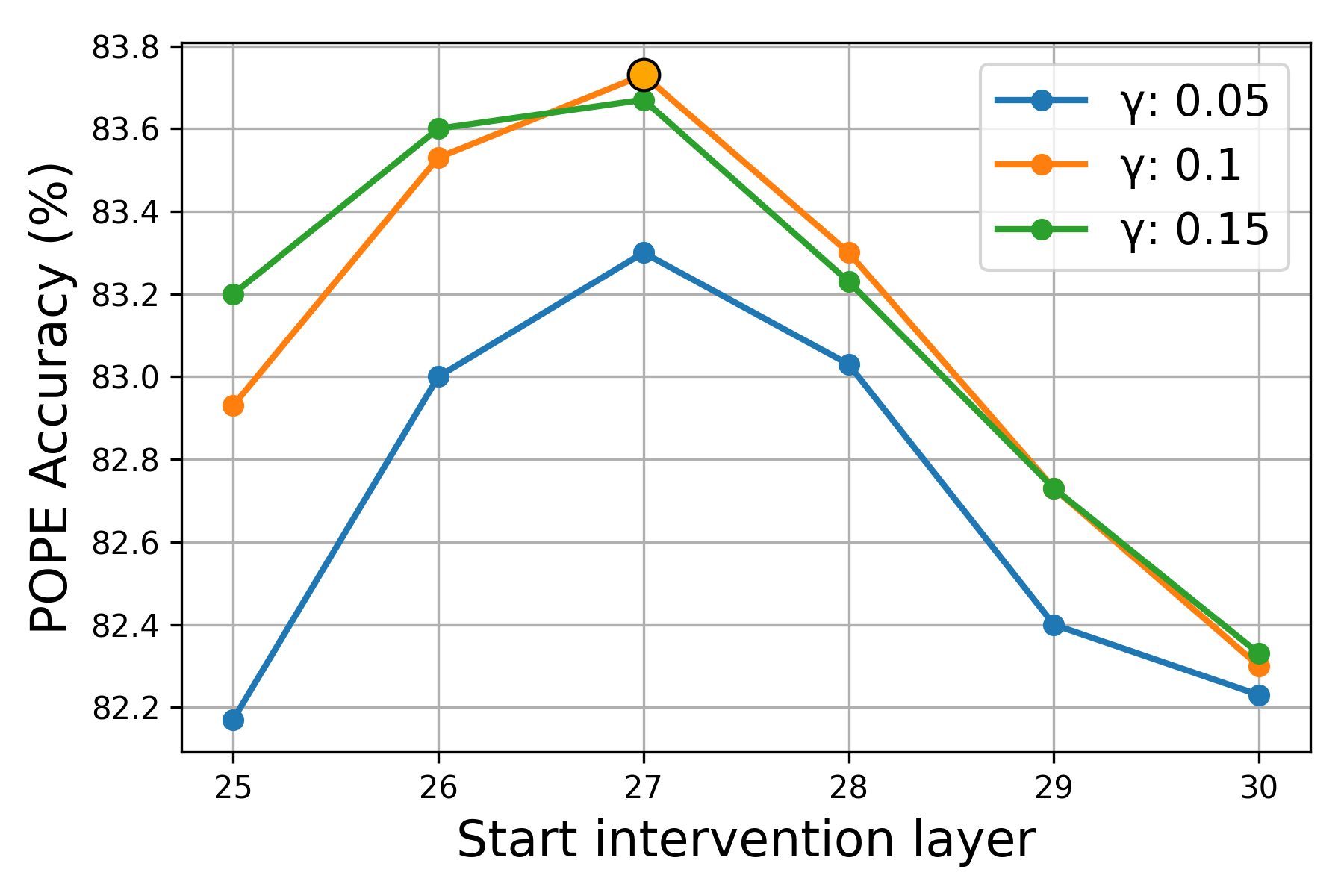}
        \caption{Ablation study on the starting intervention layer $l_s$ and the entropy threshold $\gamma$ under the random setting of A-OKVQA in POPE.}
        \label{fig:ablation_start_layer}
    \end{minipage}
    \hfill
    \begin{minipage}[t]{0.48\textwidth}
        \centering
        \includegraphics[width=\textwidth]{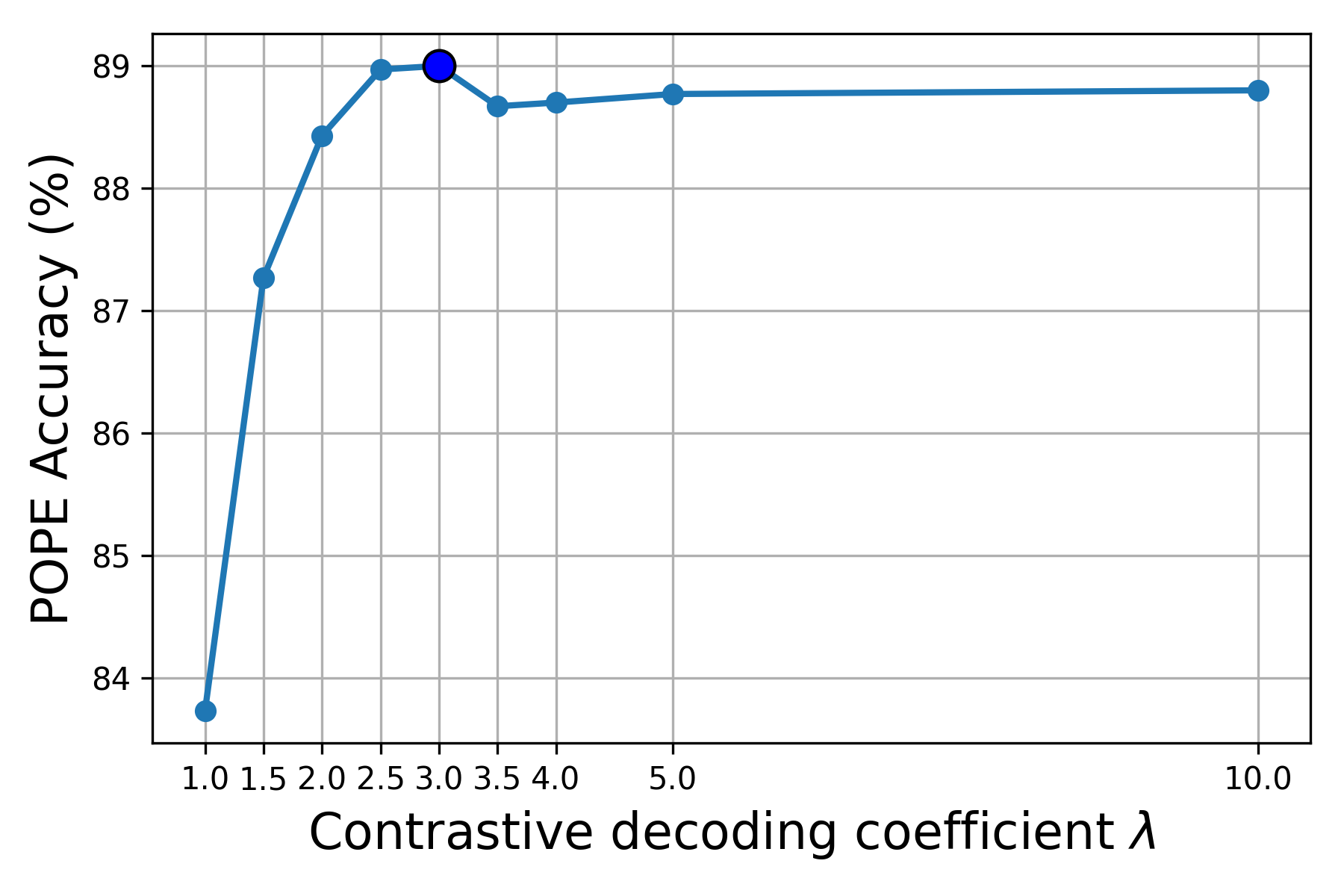}
        \caption{Ablation study on the contrastive decoding coefficient $\lambda$ under the random setting of A-OKVQA in POPE.}
        \label{fig:ablation_lambda}
    \end{minipage}
\end{figure}

\begin{figure*}[t]
    \centering
    \begin{subfigure}[b]{0.19\textwidth}
        \centering
        \includegraphics[width=\textwidth]{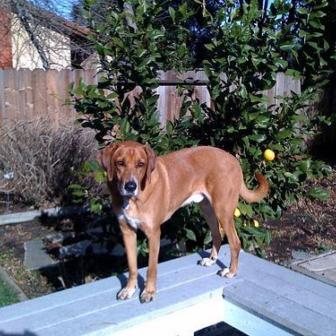}
        \caption{Image sample}
    \end{subfigure}
    \hfill
    \begin{subfigure}[b]{0.19\textwidth}
        \centering
        \includegraphics[width=\textwidth]{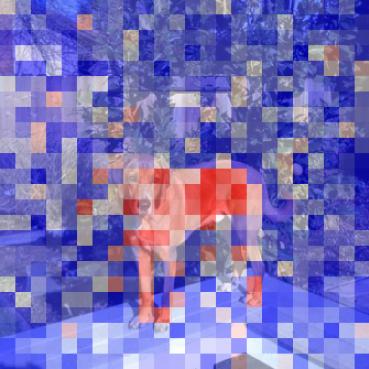}
        \caption{$0$-th layer}
    \end{subfigure}
    \hfill
    \begin{subfigure}[b]{0.19\textwidth}
        \centering
        \includegraphics[width=\textwidth]{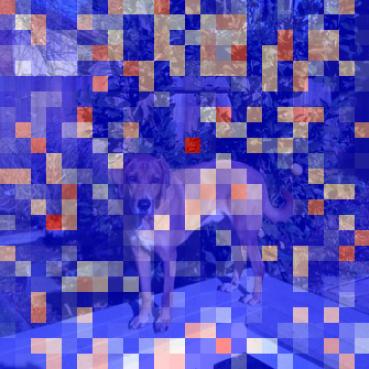}
        \caption{$11$-th layer}
    \end{subfigure}
    \hfill
    \begin{subfigure}[b]{0.19\textwidth}
        \centering
        \includegraphics[width=\textwidth]{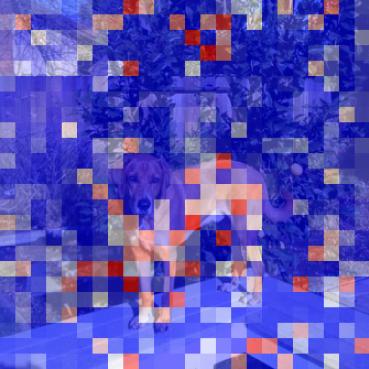}
        \caption{$21$-th layer}
    \end{subfigure}
    \hfill
    \begin{subfigure}[b]{0.19\textwidth}
        \centering
        \includegraphics[width=\textwidth]{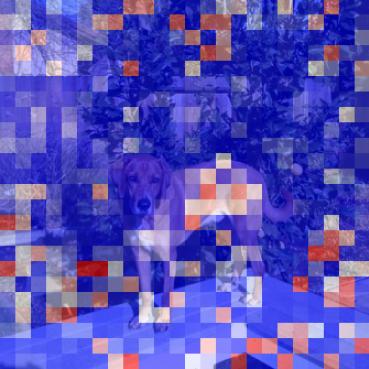}
        \caption{$31$-th layer}
    \end{subfigure}
    \caption{\textbf{Visualization of token-image similarity across decoding layers.} Given visual input (a) and the query \textit{“Is there a black dog in the image?”}. It can be observed that the hidden states in the $0$-th layer strongly correspond to \textit{`dog'}, whereas the deeper layers exhibit a diminished focus on local object features, reflecting a shift toward more global representations.}
    \label{fig:app_sim}
\end{figure*}

\begin{figure}[t]
    \centering
    \resizebox{1.\textwidth}{!}{
    \begin{tabular}{c}
        \begin{subfigure}[c]{0.3\textwidth}
            \centering
            \includegraphics[width=\textwidth]{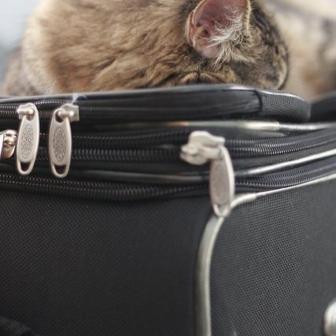}
            \caption{Image sample}
        \end{subfigure}
        \hfill
        \begin{subfigure}[c]{0.65\textwidth}
            \centering
            \includegraphics[width=\textwidth]{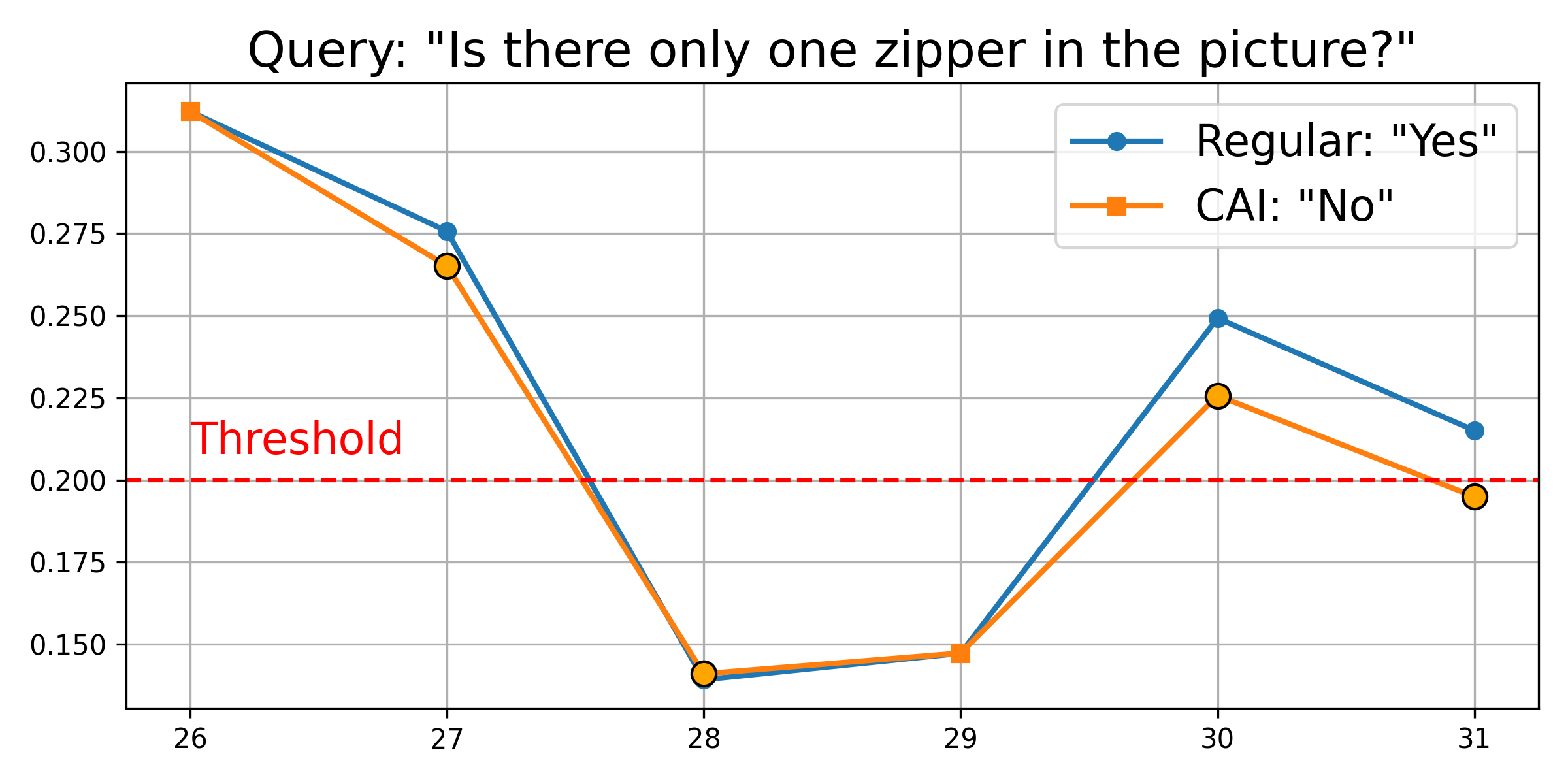}
            \caption{Entropy across layers in token generation}
        \end{subfigure}
        \\[0.7em]

        \begin{subfigure}[t]{0.3\textwidth}
            \centering
            \includegraphics[width=\textwidth]{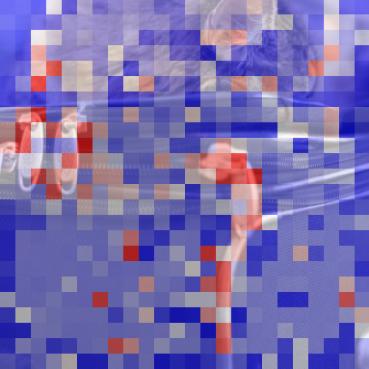}
            \caption{Token-image similarity}
        \end{subfigure}
        \hfill
        \begin{subfigure}[t]{0.3\textwidth}
            \centering
            \includegraphics[width=\textwidth]{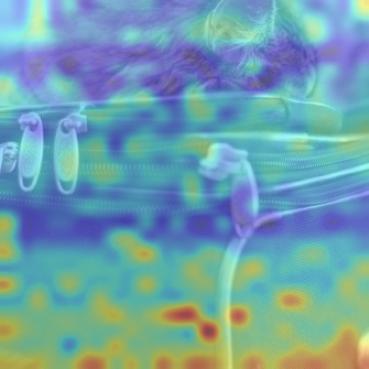}
            \caption{Before intervention (Regular)}
        \end{subfigure}
        \hfill
        \begin{subfigure}[t]{0.3\textwidth}
            \centering
            \includegraphics[width=\textwidth]{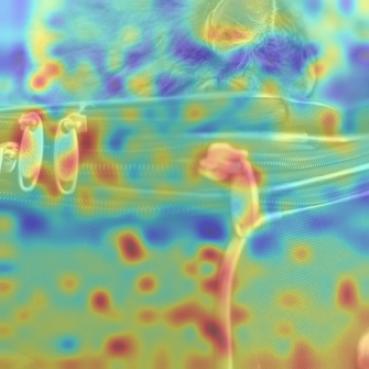}
            \caption{After intervention (CAI)}
        \end{subfigure}
    \end{tabular}
    } 
    \caption{\textbf{Case study} of visual input (a) with the query \textit{“Is there only one zipper in the picture?”}. 
    The intervention layers are shown in (b), and token-image similarity in (c) establishes grounding between the response token and the visual input. 
    The attention maps from the $0$-th head in layer $31$, before and after the similarity-guided intervention, are visualized in (d) and (e).}
    \label{fig:case}
\end{figure}

\noindent\textbf{Effect of the lowest-layer representation.}
As shown in Fig.~\ref{fig:app_sim}, the $0$-th layer hidden states exhibit a strong spatial alignment with the target object ("dog"), successfully preserving fine-grained, localized visual features; in contrast, deeper layers shift focus toward more global and abstract representations, reducing sensitivity to individual object features.
This observation aligns with recent interpretability studies on decoder-style LVLMs (e.g., LLaVA), which demonstrate that image-text communication~\cite{basu2024understanding} and object-localized information~\cite{neo2024towards} are strongest in early layers, while middle-to-deep layers increasingly aggregate visual evidence into lexical abstractions~\cite{basu2024understanding} and become more influenced by language priors~\cite{neo2024towards}.
This structural progression naturally motivates our design: use early layers to estimate where the token-relevant visual evidence is, and intervene only in deeper layers where grounding tends to degrade.


\noindent\textbf{Case study.}
Fig.~\ref{fig:case} presents a case study of MME with a visual input (a) and the query \textit{“Is there only one zipper in the picture?”}.
(b) depicts the entropy of the hallucinated response token (“Yes”) generated by LLaVA-1.5, contrasted with the non-hallucinated response token (“No”) produced by our \ours.
We set the start intervention layer $l_s=26$ and the entropy threshold $\gamma=0.2$, observing that the prediction entropy at layers $26, 27, 29, \textrm{and}\ 30$ exceeds $\gamma$, indicating a heightened risk of hallucination. Accordingly, intervention is applied at the subsequent layers, namely $27, 28, 30, \textrm{and}\ 31$.
During the intervention, token-image similarity (c) between the response token and input image (a) is computed at the lowest layer to establish a grounding baseline.
In the high-risk layers, the attention maps visualized in (d) often lack attention to the relevant regions.
By amplifying the attention weights of vision regions guided by (c), the attention map in (e) is reoriented toward token-relevant image regions and thereby mitigates hallucination.

%% file: sections/7_conclusion.tex
\section{Conclusion}
\label{sec:conclusion}

We presented Context-aware Attention Intervention (CAI), a training-free and plug-and-play approach for mitigating hallucinations in large vision-language models via a \emph{non-interference} principle: \emph{see only when needed}.
Rather than uniformly amplifying visual signals, CAI performs a \emph{minimal} and \emph{token-conditioned} attention intervention that is activated only for high-risk tokens and primarily in deeper layers where visual grounding tends to degrade, thereby strengthening grounding while preserving fluency.
We further provide a principled characterization of CAI as a KL-minimal attention reweighting, together with conditions under which small tilts improve likelihood for high-uncertainty tokens and guarantees of bounded interference when the gate is inactive.
Extensive results across multiple LVLM backbones and benchmarks demonstrate consistent hallucination reduction with lightweight inference-time overhead.

%% file: sections/8_appendix.tex
\clearpage
\appendix

\section*{Appendix}
\section{Proofs for Theoretical Results}
\label{app:theory-proofs}

\begin{proof}[Proof of Theorem~\ref{thm:kl-minimal}]
Consider 
\(
\max_{q\in\Delta^{N-1}}\ \mathbb{E}_q[s]\ \text{s.t.}\ D_{\mathrm{KL}}(q\Vert a)\le \varepsilon.
\)
The Lagrangian with multiplier $\eta\ge 0$ for the KL constraint and $\tau$ for normalization is
\[
\mathcal{L}(q,\eta,\tau)=\sum_i q_i s_i-\eta\!\left(\sum_i q_i\log\frac{q_i}{a_i}-\varepsilon\right)+\tau\!\left(\sum_i q_i-1\right).
\]
First-order optimality (KKT) gives, for each $i$,
\(
s_i-\eta\left(\log q_i-\log a_i+1\right)+\tau=0
\)
so that
\(
\log q_i=\log a_i+\tfrac{1}{\eta}s_i-\tfrac{\tau+ \eta}{\eta}
\),
hence
\(
q_i\propto a_i\,\exp(\tfrac{1}{\eta}s_i).
\)
Writing $\lambda=\tfrac{1}{\eta}\ge 0$ yields the exponential tilt
\(
q^\star(\lambda)\propto a\odot e^{\lambda s}.
\)
Since $D_{\mathrm{KL}}(\cdot\Vert a)$ is strictly convex in $q$ and the objective is linear, the solution is unique.
Moreover, with $A(\lambda)=\log\sum_i a_i e^{\lambda s_i}$, one has
\(
D_{\mathrm{KL}}(q^\star(\lambda)\Vert a)=\lambda A'(\lambda)-A(\lambda)
\)
and
\(
\frac{d}{d\lambda}D_{\mathrm{KL}}(q^\star(\lambda)\Vert a)=\lambda A''(\lambda)
=\lambda\,\mathrm{Var}_{q^\star(\lambda)}[s]\ge 0,
\)
so the KL budget $\varepsilon$ selects a unique $\lambda\ge 0$ by monotonicity. 
\end{proof}

\begin{proof}[Proof of Theorem~\ref{thm:entropy-gain}]
Let $x(a)=\sum_i a_i v_i$ and $z_y=u_y+w_y^\top x(a)$, $p=\mathrm{softmax}(z)$, and $\mathcal{L}(a)=-\log p_{y^\star}$.
The gradient of $\mathcal{L}$ w.r.t.\ $x$ is
\(
\nabla_x \mathcal{L}(a)=\sum_y p_y w_y - w_{y^\star} =: g(a).
\)
By the descent lemma for $L$-smooth functions (log-sum-exp is smooth), for $\Delta x=x(\tilde a)-x(a)$,
\[
\mathcal{L}(\tilde a)-\mathcal{L}(a)\ \le\ \langle g(a),\Delta x\rangle+\frac{L}{2}\|\Delta x\|^2,
\]
for some local $L>0$ depending on $\{w_y\}$ and $p$.
If the alignment condition is met (equivalently, $\langle -g(a),\Delta x\rangle>0$, i.e., $\langle g(a),\Delta x\rangle<0$), then the linear term is negative. 
Because the CAI tilt depends smoothly on $\lambda$ and $\Delta x=O(\lambda)$ for $\lambda\to 0$, there exists $\lambda_0>0$ such that the (negative) linear term dominates the quadratic remainder for all $0<\lambda\le \lambda_0$, implying
\(
\mathcal{L}(\tilde a)<\mathcal{L}(a).
\)
The entropy condition $H\ge H_0>0$ guarantees we are in the uncertain regime where $g(a)\neq 0$ (non-vanishing gradient), so the alignment condition is meaningful. 
\end{proof}

\begin{proof}[Proof of Theorem~\ref{thm:depth-decay}]
By assumption, the visual component evolves as 
\(
x^{(l)} = M^{(l)}x^{(l-1)}
\)
and there exists $\rho<1$ such that, for all $l\ge l_0+1$, the mixing is contractive.\footnote{If we further assume the spectral (operator) norm and that $M^{(l)}$ are normal, the bound $\|M^{(l)}\|\le \rho(M^{(l)})\le \rho$ follows from $\rho(M)\le\|M\|$. In general, the same conclusion holds under the standard bounded-mixing assumption $\|M^{(l)}\|\le \rho<1$.}
Applying sub-multiplicativity of operator norms,
\[
\begin{aligned}
\|x^{(l)}\|
&=\big\|M^{(l)}M^{(l-1)}\cdots M^{(l_0+1)}\,x^{(l_0)}\big\|\\
&\le\ \prod_{k=l_0+1}^{l}\|M^{(k)}\|\,\|x^{(l_0)}\|\\
&\le\ \rho^{\,l-l_0}\,\|x^{(l_0)}\|.
\end{aligned}
\]
Hence the visual signal decays geometrically with depth, making deeper layers yield larger marginal gains from reinforcement. 
\end{proof}

\begin{proof}[Proof of Theorem~\ref{thm:bounded-risk}]
For the tilted distribution $q^\star(\lambda)$ with log-normalizer \(A(\lambda)=\log\sum_i a_i e^{\lambda s_i}\),
\[
D_{\mathrm{KL}}(q^\star(\lambda)\Vert a)=\mathbb{E}_{q^\star(\lambda)}\!\left[\log\frac{q^\star(\lambda)}{a}\right]
=\lambda A'(\lambda)-A(\lambda).
\]
A Taylor expansion of \(A\) at \(\lambda=0\) gives
\(
A(\lambda)=\mu\lambda+\tfrac{1}{2}\sigma^2\lambda^2+O(\lambda^3)
\)
with \(\mu=\mathbb{E}_{a}[s]\) and \(\sigma^2=\mathrm{Var}_{a}[s]\).
Since \(A'(\lambda)=\mu+\sigma^2\lambda+O(\lambda^2)\),
\[
\begin{aligned}
D_{\mathrm{KL}}(q^\star(\lambda)\Vert a)
&=\lambda(\mu+\sigma^2\lambda+O(\lambda^2))-(\mu\lambda+\tfrac{1}{2}\sigma^2\lambda^2+O(\lambda^3))\\
&=\tfrac{1}{2}\sigma^2\lambda^2+O(\lambda^3).
\end{aligned}
\]
By Pinsker’s inequality,
\(
\|q^\star(\lambda)-a\|_{\mathrm{TV}}
\le \sqrt{\tfrac12 D_{\mathrm{KL}}(q^\star(\lambda)\Vert a)}
=O(\lambda),
\)
and any Lipschitz functional of \(a\) (e.g., \(w_y^\top x(a)\)) is perturbed only by \(O(\lambda)\).
Therefore, when the gate is off or \(s\) is weak (yielding small \(\lambda\)), the effect of CAI is negligible. 
\end{proof}

\section{Additional Experiments}
\subsection{Evaluation metrics}
\subsubsection{CHAIR} evaluates hallucination in image captioning by quantifying references to objects absent from the image. Using 500 randomly sampled MS-COCO images and the prompt \textit{“Please describe this image in detail”}, we compute two complementary metrics. Sentence-level CHAIR denotes the proportion of captions containing at least one hallucinated object:
\begin{equation}
    \textrm{CHAIR}_S=\frac{\left|\left\{\textrm{sentences with hallucinated object}\right\}\right|}{\left|\left\{\textrm{all sentences}\right\}\right|},
\end{equation}
and instance-level CHAIR is the proportion of hallucinated mentions among all object mentions:
\begin{equation}
    \textrm{CHAIR}_I=\frac{\left|\left\{\textrm{hallucinated objects}\right\}\right|}{\left|\left\{\textrm{all objects mentioned}\right\}\right|}.
\end{equation}

\subsubsection{POPE} serves as a binary classification task. For each image, the model is prompted with \textit{“Is [object] in this image?”} to determine object presence. Performance is measured using standard metrics: \textit{accuracy}, the proportion of correct predictions, and \textit{F1-score}, the harmonic mean of precision and recall:
\begin{equation}
    F1=\frac{2\times\textrm{precision}\times\textrm{recall}}{\textrm{precision}+\textrm{recall}}.
\end{equation}

\subsubsection{MME} quantifies fidelity via yes-or-no questions on existence, count, position, and color. For each object, the model’s responses are compared to ground truth, and the overall \textit{MME-score} aggregates correctness across all attributes, providing a fine-grained measure of object-level hallucination and semantic alignment.

\subsection{Additional Results}
\subsubsection{Comparison with more baselines.}
We extended our comparative analysis to include recent methods evaluated on LLaVA-1.5 using the POPE benchmark. As demonstrated in Table~\ref{tab:baseline}, CAI achieves the highest average accuracy and F1 scores. Importantly, this superior performance is coupled with substantial computational efficiency. The next best-performing baseline, LVLMs-Saliency, necessitates both saliency model inference and per-token attention reallocation. Consequently, it incurs a latency of 4294.36 ms/token, rendering it 33.7$\times$ slower than CAI (131.33 ms/token).
\begin{table}[h]
\centering
\caption{\textbf{Additional POPE evaluation results for LLaVA-1.5.}}
\label{tab:baseline}
\resizebox{0.95\textwidth}{!}{%
\begin{tabular}{lcccccc|ccr}
\toprule
\multirow{2}{*}{\textbf{Method}} & \multicolumn{2}{c}{Random} & \multicolumn{2}{c}{Popular} & \multicolumn{2}{c|}{Adversarial} & \multicolumn{2}{c}{\textit{Average}} & Latency $\uparrow$ \\
\cmidrule(lr){2-3}\cmidrule(lr){4-5}\cmidrule(lr){6-7}\cmidrule(lr){8-9}
 & \textit{Acc}$\uparrow$ & \textit{F1}$\uparrow$ & \textit{Acc}$\uparrow$ & \textit{F1}$\uparrow$ & \textit{Acc}$\uparrow$ & \textit{F1}$\uparrow$ & \textit{Acc}$\uparrow$ & \textit{F1}$\uparrow$ & (ms/token) \\
\midrule
\textit{LLaVA-1.5} & 85.46 & 86.07 & 81.20 & 82.22 & 75.87 & 78.27 & 80.84 & 82.19 & 89.62 \\
DoLa~\cite{chuang2024dola} & 81.77 & 82.09 & \underline{89.17} & \underline{88.54} & \textbf{86.37} & \textbf{85.98} & 85.77 & 85.54 & -\\
SID~\cite{huo2024self} & 80.33 & 81.38 & \textbf{89.47} & \textbf{89.10} & 85.87 & 85.88 & 85.22 & 85.45 & -\\
FarSight~\cite{tang2025seeing} & \underline{89.07} & \underline{89.01} & 85.63 & 86.04 & 79.30 & 81.05 & 84.67 & 85.37 & -\\
LVLMs-Saliency~\cite{zhang2026hallucination} & 89.10 & 88.20 & 87.07 & 86.01 & 83.33 & 83.12 & \underline{86.50} & \underline{85.77} & 4294.36 \\
\midrule
\cellcolor{lightgray!30}\textbf{\ours} & \cellcolor{lightgray!30}\textbf{89.24} & \cellcolor{lightgray!30}\textbf{89.14} & \cellcolor{lightgray!30}86.03 & \cellcolor{lightgray!30}85.53 & \cellcolor{lightgray!30}\underline{82.77} & \cellcolor{lightgray!30}\underline{82.73} & \cellcolor{lightgray!30}\textbf{86.68} & \cellcolor{lightgray!30}\textbf{85.80} & 131.33 \\
\bottomrule
\end{tabular}}
\end{table}

\begin{table}[!b]
\centering
\caption{\textbf{Additional CHAIR and MME evaluation results for Qwen-3.5.}}
\label{tab:qwen35}
\resizebox{0.8\textwidth}{!}{%
\begin{tabular}{lccccc}
\toprule
\multirow{2}{*}{\textbf{Method}} & \multicolumn{2}{c}{Max Token = $64$} & \multicolumn{2}{c}{Max Token = $128$} & \multirow{2}{*}{\textit{MME} $\uparrow$}\\
\cmidrule(lr){2-3} \cmidrule(lr){4-5}
 & \textit{CHAIR}$_S \downarrow$ & \textit{CHAIR}$_I \downarrow$ & \textit{CHAIR}$_S \downarrow$ & \textit{CHAIR}$_I \downarrow$\\ \midrule
\textit{Qwen-3.5} & 18.6 & 9.6 & 36.2 & 12.6 & 685.00 \\
PAI & 18.4 & 9.1 & 37.2 & 12.3 & 691.66 \\
\cellcolor{lightgray!30}\textbf{CAI}$_{\textrm{GLA-gated}}$ & \cellcolor{lightgray!30}17.2 & \cellcolor{lightgray!30}10.3 & \cellcolor{lightgray!30}35.0 & \cellcolor{lightgray!30}14.4 & \cellcolor{lightgray!30}700.00 \\
\cellcolor{lightgray!30}\textbf{\ours} & \cellcolor{lightgray!30}\textbf{16.6} & \cellcolor{lightgray!30}\textbf{8.5} & \cellcolor{lightgray!30}\textbf{29.8} & \cellcolor{lightgray!30}\textbf{11.4} & \cellcolor{lightgray!30}\textbf{720.00} \\ \bottomrule
\end{tabular}
}
\end{table}

\subsubsection{Evaluation on a stronger backbone.}
To validate generalization to newer architectures, we evaluate CAI on Qwen3.5-4B~\cite{qwen35blog}, which employs a hybrid design. Specifically, the model utilizes Softmax Attention at layers 3, 7, 11, 15, 19, 23, 27, and 31, while applying Gated Linear Attention (GLA) to all remaining layers.
Since GLA operates on recurrent states rather than explicit token-to-token attention maps, CAI's similarity-guided intervention has no direct attention map to modulate. We therefore intervene only on deep Softmax-Attention layers.
As reported in Table~\ref{tab:qwen35}, directly intervening on the GLA gates (denoted as $\text{CAI}_{\text{GLA-gated}}$) yields inconsistent gains. In contrast, applying CAI exclusively to the Softmax layers robustly improves performance on both the CHAIR and MME benchmarks, establishing a new baseline across all metrics. Notably, these evaluations utilize a single, fixed hyperparameter configuration across both datasets, without benchmark-specific re-tuning.

\subsection{Additional Discussions}
\subsubsection{Implementation details of contrastive decoding.}
Following PAI~\cite{liu2024paying}, $p(y_t|\mathbf{x})$ in Eq.~\ref{eq:decoding} is obtained by a \emph{text-only} forward pass (i.e., removing the image token; for LLaVA: text around ``\textit{$<$image$>$}'', for InstructBLIP/Qwen-VL: ``\textit{\{question\}\ Answer:}'').
This adds a small overhead of {0.04 s/token}.

\subsubsection{Efficiency of sparse token interventions.}
As shown in Fig.~\ref{fig:app_samples},  in LLaVA on the MME benchmark with $l_s=26$, \ours intervenes on only $6.67\%\sim39.38\%$ of tokens in layers 27 to 31, with interventions declining in deeper layers. This empirically validates the bounded-risk property (Theorem~\ref{thm:bounded-risk}), confirming that the method perturbs only a small subset of tokens, thereby reducing hallucinations while maintaining high inference efficiency.

\begin{figure}[h]
    \centering
    \includegraphics[width=0.65\linewidth]{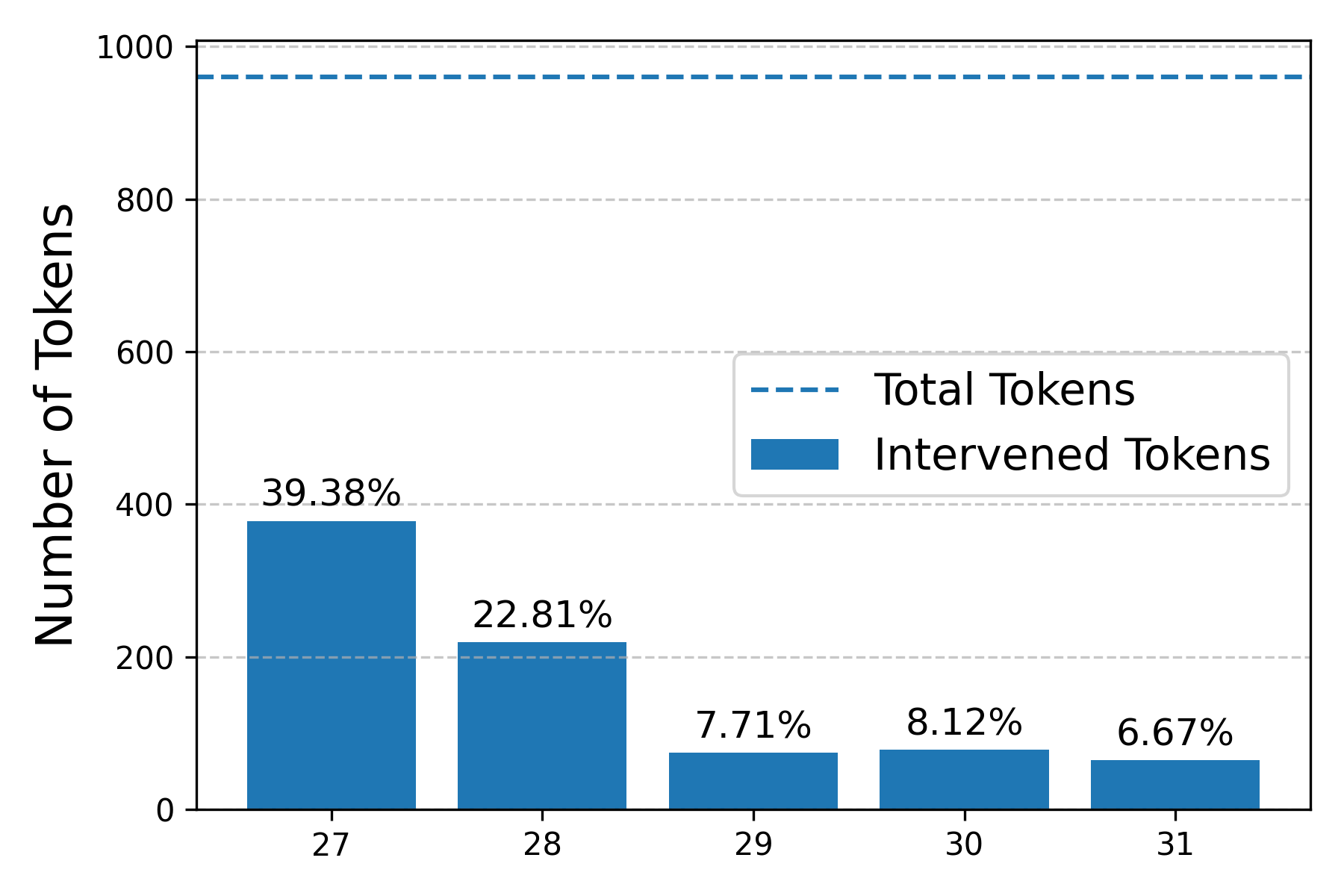}
    \caption{\textbf{Number of tokens requiring intervention at each layer}, along with their relative proportion to the total number of generated tokens.}
    \label{fig:app_samples}
\end{figure}

\subsubsection{Hyperparameter sensitivity and adaptive rules.}
CAI primarily relies on grid-search tuning for the start intervention layer $l_{s}$, the entropy threshold $\gamma$, and the contrastive decoding coefficient $\lambda$. However, fixed hyperparameters can directly transfer to newer architectures (Table~\ref{tab:qwen35}) without benchmark-specific re-tuning. To eliminate grid search entirely, we also add label-free adaptive defaults to reduce grid search: use a depth-proportional start layer
$
l_s=\lfloor 0.85L \rfloor
$
or, for hybrid architectures, the deepest Softmax-Attention layers; and use an adaptive entropy threshold
$
\gamma=\mu_H^{(t)}+0.5\sigma_H^{(t)},
$
where $\mu_H^{(t)}$ and $\sigma_H^{(t)}$ are running entropy statistics from previously decoded tokens in the same sample. This requires no validation labels. We will also clarify that contrastive decoding is optional: the attention-only CAI$^\dagger$ already improves over Regular and PAI$^\dagger$ in Table~\ref{tab:efficiency}, while $\lambda>1$ only adds complementary language-prior suppression.

\subsubsection{Additional case study.}
Fig.~\ref{fig:app_chair} illustrates three cases of long-text generation. The CHAIR metric is used to identify hallucinated and ground-truth tokens in “Regular” (LLaVA-1.5) responses. By applying \ours to reinforce visual grounding, the model amplifies the attention of response tokens toward relevant visual tokens, thereby mitigating hallucinations.
Note that, due to space constraints, only the attention map of a specific head in the intervention layer is visualized here.

\begin{figure}[t]
    \centering
    \includegraphics[width=\linewidth]{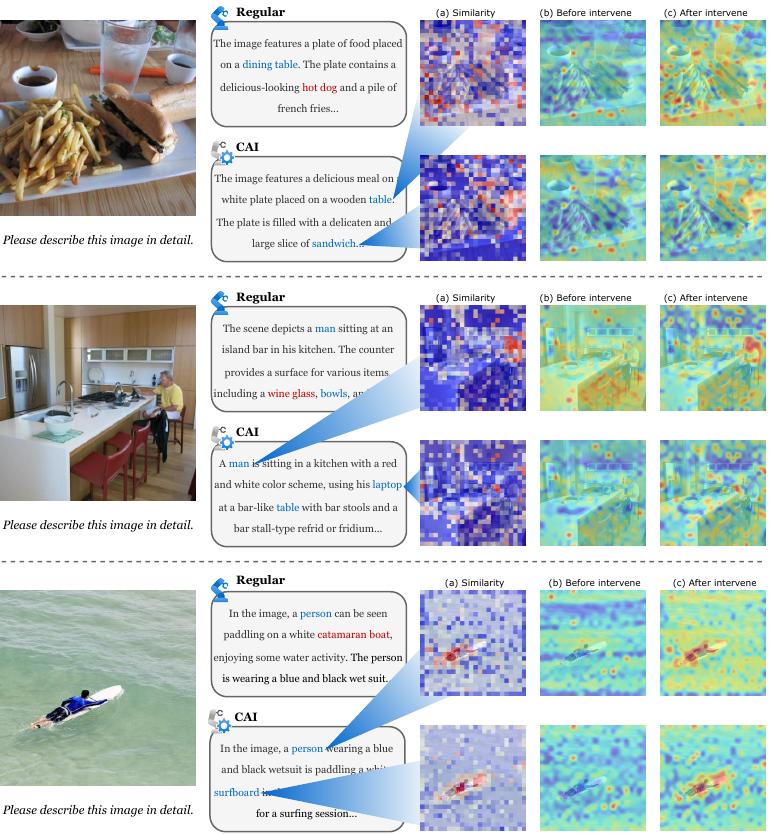}
    \caption{\textbf{Case study of CHAIR in long text generation.} Given the image and the prompt on the left, “Regular” response corresponds to LLaVA-1.5, with \colorbox{red!30}{red} denoting CHAIR hallucination words and \colorbox{blue!30}{blue} representing correctly recognized terms. Employing “CAI” effectively mitigates hallucinations, with the attention intervention process illustrated in (a), (b), and (c) on the right.}
    \label{fig:app_chair}
\end{figure}
